\documentclass[shortAfour,sageh,times]{sagej}
\usepackage{graphicx}
\usepackage{caption}
\graphicspath{
    {graphics/}
}
\usepackage{amssymb}
\usepackage{pifont}

\usepackage{amsmath,amssymb,enumerate}
\usepackage{amsthm}
\usepackage{setspace}
\usepackage{booktabs}
\usepackage[usenames,dvipsnames,svgnames,table]{xcolor}
\usepackage[bookmarks=true,colorlinks=true,pdfpagemode=UseNone,citecolor=black,linkcolor=black,urlcolor=BrickRed,pagebackref]{hyperref}
\usepackage{mathtools}
\usepackage{algorithm, algorithmicx, algpseudocode}
\usepackage{blindtext}
\usepackage{gensymb}
\usepackage{xparse}
\usepackage{lipsum}
\usepackage{graphicx}
\usepackage{subcaption}
\usepackage{soul} 
\usepackage{mathrsfs}
\usepackage[mathscr]{euscript}
\usepackage{times}
\usepackage{amsthm}
\usepackage{thmtools}
\usepackage{thm-restate}
\usepackage[nameinlink,capitalise]{cleveref}

\usepackage{mathtools}
\usepackage{multirow}
\usepackage{booktabs}

\usepackage{multirow}

\usepackage{multicol}

\usepackage{amsfonts}
\usepackage{fixltx2e}
\usepackage{cleveref}
\usepackage{balance}
\usepackage[utf8]{inputenc}
\usepackage[T1]{fontenc}
\usepackage{textcomp}
\usepackage[nolist]{acronym}
\usepackage{tabularx}
\usepackage{adjustbox}
\usepackage{textcomp}
\usepackage{placeins}
\usepackage{afterpage}

\newtheorem{theorem}{Theorem}

\theoremstyle{definition}
\newtheorem{definition}{Definition}
\newtheorem{problem}{Problem}

\newtheorem{remark}{Remark}

\definecolor{note}{rgb}{0.1,0.1,1}
\definecolor{rephase}{rgb}{0.15,0.7,0.15}
\definecolor{bag}{rgb}{0.6,0.6,0.2}
\DeclareDocumentCommand{\stateTransition}{ O{} O{} }{\boldsymbol{\Phi}_{#1}^{#2}}

\usepackage{bm}

\makeatletter
\newcommand{\mathleft}{\@fleqntrue\@mathmargin0pt}
\newcommand{\mathcenter}{\@fleqnfalse}
\makeatother

\usepackage[titletoc,page]{appendix}

\usepackage{moreverb,url} 
\usepackage{secdot}
\sectiondot{subsection}    
\sectiondot{subsubsection}
 
\newcommand{\method}{\textsc{LieIPM}\xspace}
\newcommand{\methodlong}{Lie Group Interior Point Method\xspace}

\def\volumeyear{2018}

\begin{document} 
 
\runninghead{Teng et al.}

\title{\method: \methodlong for Direct Trajectory Optimization of Rigid Bodies}

\author{Sangli Teng, Ruiqi Zhang, Tzu-Yuan Lin, William A Clark, Mark Mueller, Ram Vasudevan, Maani Ghaffari\footnotemark[1], Koushil Sreenath\footnotemark[1]}


\affiliation{
S.~Teng, R.~Zhang, M.~Mueller, and K.~Sreenath are with the University of California, Berkeley. T.~Lin is with MIT. W.~Clark is with the Ohio University. R.~Vasudevan and M.~Ghaffari are with the University of Michigan, Ann Arbor.}


\begin{abstract}
Designing dynamically feasible trajectories for rigid bodies is a fundamental problem in robotics. While direct methods are widely used, the existing constrained optimizers typically operate in Euclidean space and ignore the manifold structure of rigid body motions. 
This \emph{mismatch} may introduce singularities or lead to poorly conditioned optimization problems.
To bridge this gap, we develop a structure-aware framework for constrained trajectory optimization directly on matrix Lie groups. Our approach is based on the second-order rigid body models utilizing Lie group structures, which enables efficient Newton-type updates while preserving the underlying geometry. Building on this model, we propose a line-search \methodlong (\method) to handle constraints on the manifolds. We instantiate the framework for rigid body motion planning using Lie group variational integrators and derive closed-form intrinsic derivatives that exploit group symmetries. The \method preserves the topology of rotation motions by construction and avoids singularities. Numerical results demonstrate superior robustness and faster convergence compared to general-purpose solvers and structure-exploiting optimal control methods.


\end{abstract}


\maketitle
\footnotetext[1]{M.~Ghaffari and K.~Sreenath equally advised this work.}
\begin{abstract}
    

Designing dynamically feasible trajectories for rigid bodies is a fundamental problem in robotics.
Although direct trajectory optimization is widely applied to solve this problem, inappropriate parameterizations of rigid body dynamics often result in slow convergence and violations of the intrinsic topological structure of the rotation group.
This paper introduces a Riemannian optimization framework for direct trajectory optimization of rigid bodies. 
We first use the Lie Group Variational Integrator to formulate the discrete rigid body dynamics on matrix Lie groups.
We then derive the closed-form first- and second-order Riemannian derivatives of the dynamics. 
Finally, this work applies a line-search Riemannian Interior Point Method (RIPM) to perform trajectory optimization with general nonlinear constraints.
As the optimization is performed on matrix Lie groups, it is correct-by-construction to respect the topological structure of the rotation group and be free of singularities.
The paper demonstrates that both the derivative evaluations and Newton steps required to solve the RIPM exhibit linear complexity with respect to the planning horizon and system degrees of freedom.
Simulation results illustrate that the proposed method is faster than conventional methods by an order of magnitude in challenging robotics tasks.

\end{abstract}




\section{Introduction}

\begin{figure*}
    \centering
    \includegraphics[width=1\linewidth]{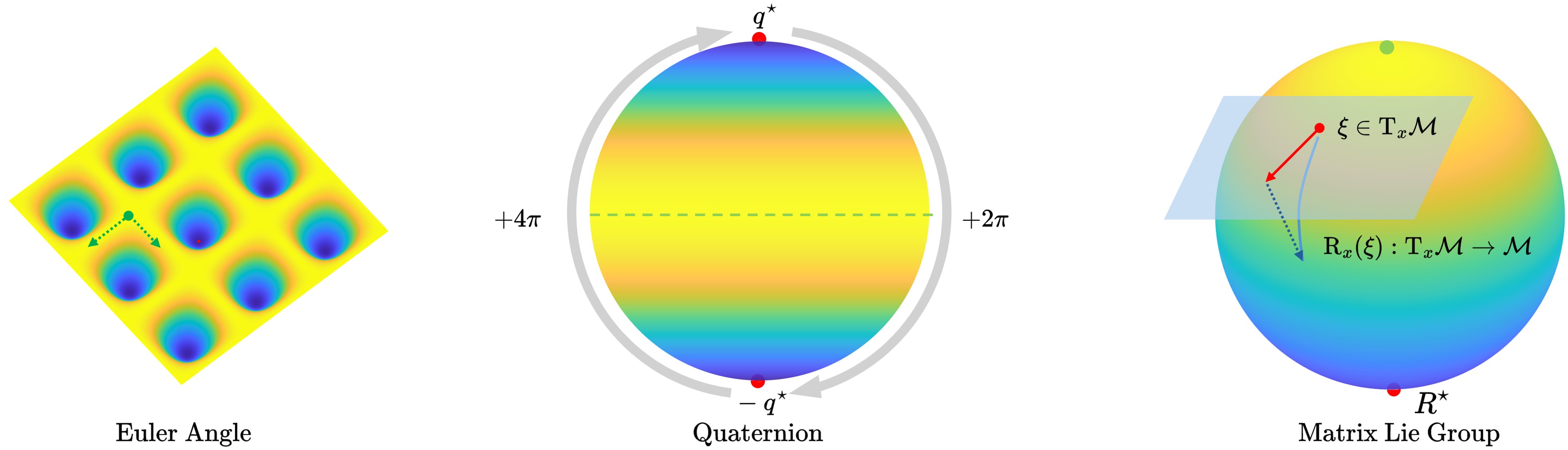}
    \caption{We develop \methodlong (\method) for direct trajectory optimization of rigid bodies on matrix Lie groups. We compare the landscape of $\min_R \|R-I\|_F^2, \mathrm{s.t}\ {R} \in \mathrm{SO}(3)$, using different parameterizations. For the Euler angle with the yaw angle fixed, we see repeated degenerate directions as a limit of the 3 DOF parameterization of $\mathrm{SO}(3)$. The equivalent quaternion version has two distinct solutions as a consequence of the double cover issue. Only the matrix Lie groups version, plotted on the quotient space $\mathrm{S}^2 \simeq \mathrm{SO}(3)/\mathrm{SO}(2)$, has a unique solution.  
    }
    \label{fig:teaser}
    \vspace{-4mm}
\end{figure*}


Optimizing the rigid body motions that evolve on a manifold is a fundamental problem in robotic control and motion planning.
Consider the Direct Trajectory Optimization (DTO) problem of a rigid body evolving in 3D space:
\begin{equation}
\label{eq:TO}
\begin{aligned}
    \min_{x_k,u_k} \quad &\ell_T(x_T) + \sum_{k=0}^{T-1}\ell(x_k,u_k) \\
    \mathrm{s.t.} \quad & f_d(x_{k+1}, x_k, u_k) = 0 \\
    \quad & g(x_k, u_k) \le 0 \\
    & x_k \in \mathcal{M}_{RB}
\end{aligned}\tag{DTO}
\end{equation}
with $\ell_T(\cdot)$ the terminal cost, $\ell(\cdot, \cdot)$ the running cost,  $f_d$ the implicit function of the discretized dynamics, $u_k$ the control input, $g(\cdot, \cdot)$ the inequality constraints, and $x_k \in \mathcal{M}_{RB}$ the state of rigid body defined on a non-Euclidean configuration space $\mathcal{M}_{RB}$ such as $\mathrm{SE}(3)$, $\mathrm{SO}(3)\times \mathbb{R}^3$, or its equivalent quaternion representations. 

The~\ref{eq:TO} has been widely studied for robot motion planning~\citep{hargraves1987direct, betts1998survey, posa2014direct, hereid2017frost}, where most existing approaches model robot dynamics in the generalized coordinates~\citep{bullo1999tracking}. When the derivatives are properly derived,~\ref{eq:TO} can be solved efficiently by nonlinear optimization techniques~\citep{wright2006numerical}.


However, the numerical optimization developed on Euclidean space~\citep{wright2006numerical} has fundamental limits when the decision variables are on manifolds with non-trivial topology. As the real projective space $\mathbb{RP}^3$ cannot be embedded differentiably into $\mathbb{R}^3$ \citep{Lee2003}, there does not exist a globally smooth mapping from $\mathbb{R}^3$ to $\mathrm{SO}(3)$ (rotation matrices) or $\mathrm{SU}(2)$ (unit quaternions). Thus, any effort trying to parameterize 3D rotational motions by $\mathbb{R}^3$, such as Euler angles, three Degree-Of-Freedom (DOF) parameterization of quaternion \citep{brudigam2021integrator} or Rodriguez formula \citep{kalabic2017mpc}, introduces singularities. As a consequence, solving \ref{eq:TO} with an off-the-shelf nonlinear optimization solver either suffers from the singularities in minimal coordinates or needs to incorporate additional constraints in the ambient space that decelerate the solutions. Though the quaternion can avoid the singularities, its topology is not globally consistent with $\mathbb{RP}^3$, thus having the double cover issue. While $\mathrm{SO}(3)$ has the same topology as $\mathbb{RP}^3$, each pose is uniquely represented by a rotation matrix. Thus we hypothesize that: {\emph{solving \eqref{eq:TO} on matrix Lie groups is topologically correct and will lead to superior computational efficiency.}} The comparison of different parameterization is illustrated in~\Cref{fig:teaser}. 



In light of such topological advantages, we describe rigid body dynamics by the Euler-Poincaré equations~\citep{marsden1998introduction, bloch2003nonholonomic} on matrix Lie groups. 
Despite the potential advantages, solving~\ref{eq:TO} on matrix Lie groups requires \emph{constrained} optimization on manifolds, which has not been as well studied as the \emph{unconstrained} ones in 3D perception~\citep{forster2016manifold, clark2021nonparametric, rosen2019se, han2025building}. 

To address the gap, we develop the \methodlong(\method) to solve~\ref{eq:TO} on matrix Lie groups. 
An early version of this work by the current authors was
presented as a conference paper in \citep{TengS-RSS-25}. 
The \emph{main contributions} of this work are: 
\begin{enumerate}
    \item We formulate trajectory optimization of rigid bodies as a constrained optimization problem on matrix Lie groups, preserving the geometric structure by construction.
    
    \item We derive a consistent closed-form second-order model of the dynamics based on Lie group symmetry, enabling efficient Newton-type optimization on the manifold.
    
    \item We develop the \methodlong (\method) for solving general constrained optimization problems directly on matrix Lie groups.
    
    \item We validate the proposed framework on trajectory optimization tasks, demonstrating superior robustness and convergence compared to existing methods.

    \item Open-source C++ implementation at \url{https://github.com/SangliTeng/LieIPM}. 
\end{enumerate}
{\color{black}The \emph{new} contributions compared to the conference version ~\citep{TengS-RSS-25} include:}
\begin{enumerate}
    \item Enhanced line search and inertia correction scheme to improve the robustness of \method.
    \item Convergence analysis of the proposed solver. 
    \item An extension of the proposed method to non-compact matrix Lie groups, where the Riemannian exponential and Lie exponential are not identical. 
    \item Deployment of the proposed method on real-world experiments.
\end{enumerate}

The remainder of the paper is organized as follows. The related work is summarized in \Cref{sec:related-work}. The math preliminary is provided in \Cref{sec:prelim}. The rigid body dynamics and its differentiation are presented in \Cref{sec:srb-dynamics} and \ref{sec:diff-srb-dynamics}, respectively. The \method is formulated in \Cref{sec:RIPM} and evaluated in \Cref{sec:num-exp}. Finally, the limitations of the proposed method is discussed in \Cref{sec:lim} and the conclusions are summarized in \Cref{sec:conclusion}. 
\section{Related Work}
\label{sec:related-work}
In this section, we review the trajectory optimization of rigid body systems and optimization on manifolds.
\subsection{Rigid Body Dynamics}
The majority of robotics applications model the dynamics in generalized coordinates. By expressing the kinetic energy in terms of joint angles, the robot dynamics evolves on the Riemannian manifold with the metric defined by the inertia~\citep{bullo2019geometric}. Based on this model, the tracking controller \citep{bullo1999tracking} and its variants in task space \citep{ratliff2018riemannian, sentis2005synthesis, khatib1987unified} can be applied for feedback control of robots. On the other hand, the Euler-Poincaré equation~\citep{marsden1998introduction, bloch2003nonholonomic} models the rigid bodies on matrix Lie groups, which is correct-by-construction to respect the topological structure of the rotation group.

The linearized single rigid body dynamics has been applied for tracking control of legged robots \citep{kim2019highly, ding2021representation, agrawal2022vision, teng2022error}. The simplified models with dynamics of angular rates neglected are applied in trajectory generation of the quadrotor \citep{mueller2013computationally}. Though the single rigid models are sometimes simplified, they are representative in many scenarios for robot motion planning. 
\subsection{Trajectory Optimization}
Trajectory optimization aims to synthesize robot motions subject to dynamics, kinematics, input, and environment constraints. The direct collocation method derives the dynamics in discrete time and then conducts optimizations \citep{zucker2013chomp, schulman2014motion, posa2014direct, hereid2017frost, manchester2019contact, Dong-RSS-23, li2023autonomous}. By the state-of-the-art numerical optimization techniques \citep{wachter2006implementation, gill2005snopt}, the direct methods can handle large-scale problems with complicated constraints. 


%

The full rigid body dynamics is incorporated to quaternion-based MPC for real-time control of quadrotor~\citep{nan2022nonlinear, sun2022comparative,sun2025agile}. Due to the manifold nature of rotations, the quaternion-based MPC does not fully respect the manifold constraints, as its integrators are designed in the ambient space without the algebraic constraints of quaternions. To handle full rotation dynamics with correct manifold constraints in discrete time, the variational integrator-based~\citep{marsden1998introduction, lee2005lie} has been applied to~\ref{eq:TO}. By the variational integrator, there is no need to integrate the continuous-time vector field to obtain the on-manifold states that introduce the singularities~\citep{teng2024convex}. The OCP with discrete dynamics on the Lie group is formulated in \citep{kobilarov2011discrete} and can be solved via iterative root finding to meet the first-order optimality condition. The on-manifold Model Predictive Control (MPC) \citep{kalabic2017mpc} synthesized the trajectory with full dynamics on $\mathrm{SO}(3)$ and applied the numerical optimization by lifting the variable to the Lie algebra \citep{lee2005lie}. As \citep{boutselis2020discrete, kobilarov2011discrete, kalabic2017mpc} are designed on Lie groups, they do not face the singularity of gimbal locks when using Euler angles. However, the Rodrigues formula \citep{kalabic2017mpc}, a three-DOF parameterization of the rotational map, inevitably introduces singularities. Though such singularities can sometimes be avoided by manually rounding the velocity vector \citep{wang2025unlocking}, an optimization framework that intrinsically handles the topological structures of the rotation group remains absent. Other than the direct collocation method, the shooting-based method, such as the Differential Dynamic Programming (DDP) is applied in \citep{boutselis2020discrete} to synthesize an optimal trajectory on the matrix Lie groups. However, its dynamics are formulated by integrating the continuous-time vector field, which does not have a closed-form solution. 

An alternative solution to trajectory optimization on manifolds is to leverage the embedding theorem~\citep{hirsch2012differential} to formulate the optimizations in the ambient space~\citep{bonalli2019trajectory}. However, such an extrinsic representation inevitably increases the size of the optimization and still requires the mapping from the velocities in the tangent space to the manifold that can potentially introduce singularities. 

\subsection{Optimization on Manifolds}
On-manifold optimization has been extensively applied to problems involving variables defined on smooth manifolds \citep{boumal2023introduction, absil2008optimization}, such as matrix completion \citep{vandereycken2013low} and Semi-Definite Programming (SDP) \citep{wang2023solving, burer2003nonlinear}. With proper on-manifold derivatives and the retraction map, one can seamlessly extend traditional first- and second-order numerical optimizations to the manifold setting. Since geometric optimization inherently finds search directions on the manifold, it involves far fewer constraints than formulations in the ambient space. For example, the on-manifold Gauss-Newton method has been successfully applied in SLAM \citep{forster2016manifold} on $\mathrm{SE}(3)$. Similarly, optimization on the Stiefel manifold has been proposed to solve the relaxed pose graph optimization to obtain globally optimal solutions \citep{rosen2019se}. Furthermore, Riemannian optimization has also demonstrated superior convergence rates in continuous sensor registrations \citep{clark2021nonparametric}.

Beyond perception tasks, OCPs for rigid bodies have been modeled on matrix Lie groups~\citep{teng2024convex, Teng-RSS-23, teng2022lie, teng2022error, 10301632} and could potentially benefit from on-manifold optimization techniques. The symmetry of Lie groups has been widely applied in observer design~\citep{barrau2016invariant, hartley2020contact, van2022equivariant} and feedback control~\citep{teng2021legged, teng2022error}. More recent work has also leveraged the manifold structure in learning to enhance the fidelity in Hamiltonian~\citep{duong2022adaptive, duong2024port} and hybrid systems~\citep{teng2025chyll}.

However, there is an absence of constrained on-manifold optimization solvers tailored for motion planning of rigid bodies on matrix Lie groups. To address this gap, the application of \emph{constrained} Riemannian optimization \citep{schiela2020sqp, obara2022sequential, yamakawa2022sequential, liu2020simple, lai2024riemannian} holds significant promise to consider general nonlinear constraints on matrix Lie groups. In this work, we developed the manifold extension of the interior point method solvers for \ref{eq:TO}.




\section{Math Preliminaries}
\label{sec:prelim}

In this section, we review the geometric ingredients to define second-order optimization methods on manifolds, followed by their specialization to matrix Lie groups. 

\subsection{Second-order Model on Manifolds}

Let $\mathcal{M}$ be a finite-dimensional smooth manifold. We denote the tangent (resp. co-tangent) space at $x\in \mathcal{M}$ by $\mathrm{T}_x\mathcal{M}$ (resp. $\mathrm{T}^*_x\mathcal{M}$) and the tangent bundle by $\mathrm{T}\mathcal{M}$ (resp. $\mathrm{T}^*\mathcal{M}$). A vector field is a map $\zeta:\mathcal{M}\to \mathrm{T}\mathcal{M}$ with $\zeta(x)\in \mathrm{T}_x\mathcal{M}$, and we denote the set of smooth vector fields by $\mathfrak{X}(\mathcal{M})$.

\begin{definition}[Affine Connection]
An affine connection on $\mathcal{M}$ is a $\mathbb{R}$-bilinear map
$
\nabla:\mathfrak{X}(\mathcal{M})\times\mathfrak{X}(\mathcal{M})\to\mathfrak{X}(\mathcal{M})
$
such that the following conditions are satisfied for $X,Y \in \mathfrak{X}(\mathcal{M})$ and any smooth function $f:\mathcal{M} \rightarrow \mathbb{R}$:
\begin{enumerate}
    \item $\nabla_{fX}Y = f\nabla_XY$, and
    \item $\nabla_X(fY) = \mathcal{L}_X f \cdot Y + f\nabla_XY$,
\end{enumerate}
\end{definition}

The affine connection defines directional derivatives of vector fields and covectors, and thus provides the basic ingredient for second-order models on $\mathcal{M}$.

\begin{definition}[Derivative and Hessian]
Let $f:\mathcal{M}\to\mathbb{R}$ be smooth. Its derivative at $x\in\mathcal{M}$ is the covector
\begin{equation}
    \mathrm{D}f(x) \in \mathrm{T}_x^*\mathcal{M},
\end{equation}
whose action on $\xi\in \mathrm{T}_x\mathcal{M}$ is denoted by the differential $\mathrm{D}f(x)[\xi]$.
Given an affine connection $\nabla$, the Hessian of $f$ at $x$ is the bilinear form
\begin{equation}
    \mathrm{Hess}\, f(x)[\xi_1,\xi_2] := (\nabla_{\xi_1} \mathrm{D}f)_x[\xi_2],
    \ \xi_1,\xi_2\in \mathrm{T}_x\mathcal{M}.
\end{equation}
\end{definition}

\begin{definition}[Second-order Model]
Let $f:\mathcal{M}\to\mathbb{R}$ be smooth, let $x\in\mathcal{M}$ and $\xi \in \mathrm{T}_x\mathcal{M}$, and consider the retraction curve $c(t)$ with $c(0)=x$ and $\dot{c}(0)=\xi$. Then the second-order expansion of $f$ along $c$ is
\begin{equation}
\begin{aligned}
    f(c(t)) &= f(x)
    + t\, \mathrm{D}f(x)[\xi] \\
    \quad + \frac{t^2}{2}&\left(
    \mathrm{Hess}\, f(x)[\xi,\xi]
    + \mathrm{D}f(x)\!\left[\frac{D\dot c}{dt}(0)\right]
    \right) + \mathcal{O}(t^3).
\end{aligned}
\end{equation}
\end{definition}

To map the tangent vector $\xi \in \operatorname{T}_x\mathcal{M}$ to the $\mathcal{M}$, we have the retraction map:
\begin{definition}[Retraction] A retraction on $\mathcal{M}$ is a smooth map
\begin{equation}
    \mathrm{R}: \mathrm{T}\mathcal{M} \to \mathcal{M},\quad (x,\xi)\mapsto \mathrm{R}_x(\xi),
\end{equation}
such that the retraction curve $c(t) = \mathrm{R}_x(t\xi)$ satisfies $c(0)=x$ and $\dot c(0)=\xi$.
\end{definition}

Given a retraction, the curve $c(t)=\mathrm{R}_x(t\xi)$ provides a local parametrization around $x$ along which a second-order expansion can be formed.

\subsection{Newton's Method on Manifolds}

We now formulate Newton's method for solving nonlinear equations on manifolds. Let $\zeta$ be a smooth vector field:
\begin{equation}
    \zeta : \mathcal{M} \to \mathrm{T}\mathcal{M},
    \qquad
    \zeta(x) \in \mathrm{T}_x\mathcal{M}.
\end{equation}
We consider the problem of finding a root of $\zeta$, namely $    \zeta(x) = 0.$
Given an affine connection $\nabla$, the Jacobian of $\zeta$ at $x$ is the linear map
\begin{equation}
    J_\zeta(x) : \mathrm{T}_x\mathcal{M} \to \mathrm{T}_x\mathcal{M},
\end{equation}
defined by $J_\zeta(x)\xi := \nabla_{\xi}\zeta,\ \xi \in \mathrm{T}_x\mathcal{M}.$
\begin{definition}[Newton Step]
Let $x_k\in \mathcal{M}$. The Newton direction $\xi_k \in \mathrm{T}_{x_k}\mathcal{M}$ is defined as the solution of
\begin{equation}
    J_\zeta(x_k)\xi_k = -\zeta(x_k).
\end{equation}
The next iterate is given by $x_{k+1} = \mathrm{R}_{x_k}(\xi_k)$, where $\mathrm{R}$ is a retraction.
\end{definition}

The on-manifold version of Newton iterations exhibits similar local convergence as in Euclidean space under proper conditions:
\begin{theorem}[Local Quadratic Convergence~\citep{absil2008optimization}]
\label{theorem:quad-newton}
Let $\zeta$ be a smooth vector field on $\mathcal{M}$, and let $x^\star \in \mathcal{M}$ be a non-degenerate root, i.e.,
\begin{equation}
    \zeta(x^\star)=0,
    \qquad
    J_\zeta(x^\star) \text{ is nonsingular}.
\end{equation}
Then there exists a neighborhood $\mathcal{U}$ of $x^\star$ such that, for any initial point $x_0 \in \mathcal{U}$, the Newton iteration generates a sequence $\{x_k\}$ that converges to $x^\star$ at least quadratically.
\end{theorem}


\subsection{Matrix Lie Groups}

Let $\mathcal{G}$ be an $n$-dimensional matrix Lie group and $\mathfrak{g}$ the associated Lie algebra, i.e, the tangent space of $\mathcal{G}$ at the identity. For convenience, we define the following isomorphism 
\begin{equation}
    (\cdot)^\wedge:\mathbb{R}^n \rightarrow \mathfrak{g}, \quad (\cdot)^\vee:\mathfrak{g} \rightarrow \mathbb{R}^n. 
\end{equation}
that maps between the vector space $\mathbb{R}^n$ and $\mathfrak{g}$. Then, $\forall \phi \in \mathbb{R}^{n}$, we can define the Lie exponential map as
\begin{equation}
    \exp(\cdot):\mathbb{R}^{n} \rightarrow \mathcal{G},\ \ \exp(\phi)=\operatorname{exp_m}({\phi}^\wedge),
\end{equation}
where $\operatorname{exp_m}(\cdot)$ is the exponential of square matrices. We also define the Lie logarithmic map as the inverse of the Lie exponential map in the local branch:
\begin{equation}
    \log(\cdot): \mathcal{G} \rightarrow \mathbb{R}^{n}. \ \ 
\end{equation}
For every $X \in \mathcal{G}$, the adjoint action, $\mathrm{Ad}_{X}: \mathfrak{g}\rightarrow \mathfrak{g}$, is a Lie algebra isomorphism that enables change of frames:
\begin{equation}
    \mathrm{Ad}_{X}({\phi}^\wedge)= X{{\phi}^\wedge}X^{-1}.
\end{equation}
Its derivative at the identity gives rise to the adjoint map in Lie Algebra as
\begin{equation}
    \mathrm{ad}_{\phi}(\eta) = [{\phi}, {\eta}],
\end{equation}
where $\phi^\wedge, \eta^\wedge \in \mathfrak{g}$ and $[\cdot, \cdot]$ is the Lie bracket such that
\begin{equation}
    [\phi, \eta] = (\phi^\wedge\eta^\wedge - \eta^\wedge\phi^\wedge)^\vee.
\end{equation}

Due to the possible non-commutativity of $\mathcal{G}$, we have the BCH formula~\citep{hall2013lie} to compute $\xi_3$ given $\exp(\xi_3) = \exp(\xi_1)\exp(\xi_2)$ with $\xi_1^\wedge, \xi_2^\wedge,\xi_3^\wedge \in \mathfrak{g}$:
\begin{equation}
\begin{aligned}
        \xi_3 = & \xi_1 + \xi_2 + \frac{1}{2}[\xi_1, \xi_2] \\
        + \frac{1}{12}(& [\xi_1, [\xi_1, \xi_2]] + [\xi_2, [\xi_2, \xi_1]]) \cdots
\end{aligned}\tag{BCH}
\end{equation}

Now we proceed to define the second-order models on the matrix Lie group based on these operations. For $X\in\mathcal{G}$ and $\xi\in\mathfrak{g}$, consider the curve
\begin{equation}
    c(t)=X\exp(t\xi),
\end{equation}
which induces the map $\mathrm{R}_X(\xi)=X\exp(\xi)$. 
Since $\exp(0)=I$ and $\frac{d}{dt}\exp(t\xi)\big|_{t=0}=\xi$, it follows that
$\mathrm{R}_X$ satisfies the retraction conditions, and thus defines a valid local parametrization.

To construct a second-order model on $\mathcal{G}$, we need to equip it with a proper affine connection. 
We choose a connection to be compatible with the Lie exponential retraction, so that the resulting Hessian provides the correct second-order expansion along the update curve $X\exp(t\xi)$: 

\begin{definition}[Second-order Models on Matrix Lie Groups~\citep{mahony2002geometry}]
\label{def:2nd-lie}
Consider a smooth function $f:\mathcal{G}\rightarrow \mathbb R$ and equip $\mathcal{G}$ with an Cartan-Schouten connections, we have $\frac{D\dot c}{dt}(0)=0,$ and the second-order expansion reduces to
\begin{equation}
\begin{aligned}
    f(X\exp(t\xi))
    &=
    f(X)
    + t\,\mathrm{D}f(X)[\xi]  \\
    \quad
    + \frac{t^2}{2}&\,
    \mathrm{Hess}\, f(X)[\xi,\xi]
    + \mathcal{O}(t^3).
\end{aligned}
\end{equation}
\end{definition}
Therefore, Newton's method updates the state by
\begin{equation}
    J_\zeta(X_k)\,\delta \xi_k = -\zeta(X_k),\ X_{k+1}=X_k\exp(\delta \xi_k)
\end{equation}
with $\delta\xi_k \in \mathrm{T}_{x_k}\mathcal{M}$ and retains local superlinear convergence under the non-degenerate conditions in~\Cref{theorem:quad-newton}.

For the choice of an affine connection, including the Levi-Civita connection that is standard in Riemannian optimization~\citep{boumal2014manopt} and Cartan-Schouten connection that is more natural on matrix Lie groups~\citep{mahony2002geometry}, see Appendix.~\ref{appx:connection}.
\section{Discrete Rigid Body Dynamics}
\label{sec:srb-dynamics}
In this section, we introduce the rigid body dynamics for~\eqref{eq:TO} based on the variational integrators on $\mathcal{M}_{\mathrm{RB}}$, i.e., the configuration states in \eqref{eq:TO}. As discussed in~\citep{teng2024convex}, the variational integrator is the only integrator that preserves the manifold structure in discrete time without the constraints in the ambient space.  

\subsection{Variation-based Discretization}

Consider a mechanical system with the configuration space $\mathcal{M}$. We denote the configuration state as $x \in \mathcal{M}$ and the generalized velocity as $\dot{x} \in \operatorname{T}_{x}\mathcal{M}$. Then we have the Lagrangian given the kinetic and potential energy $T(\dot{x}), V(x)$:
\begin{equation}
\label{eq:ct_lag}
    L(x, \dot{x}):=T(\dot{x}) - V(x). 
\end{equation}
The key idea of a variational integrator is to discretize the Lagrangian \eqref{eq:ct_lag} to obtain the discrete-time equation of motion (EoM) \citep{marsden2001discrete}. The discretization scheme ensures that the Lagrangian is conserved in discrete time, thus having superior energy conservation properties over long durations. 
The discrete Lagrangian $L_d: \mathcal{M}\times \mathcal{M} \rightarrow \mathbb{R}$ could be considered as the approximation of the action integral via: 
\begin{equation}
\label{eq:lag_ct2dt}
    L_d(x_k, x_{k+1}) \approx \int ^{t_{k+1}}_{t_k}L(x, \dot{x})dt .
\end{equation}
Then the discrete variant of the action integral becomes:
\begin{equation}
    S_d = \sum_{k=0}^{N-1}{L}_d(x_k, x_{k+1}) .
\end{equation}
Finally, we take variation in $\operatorname{T}\mathcal{M}$ and group the term corresponding to $\delta x_{k} \in \operatorname{T}_{x_k}\mathcal{M}$ as the discrete version of integration by parts \citep{marsden2001discrete}: 
\begin{equation}
\begin{aligned}
  \delta S_d &= D_1 {L}_d(x_0, x_1)[\delta x_0] +  D_2 {L}_d(x_{N-1}, x_N)[\delta x_N]\\ 
    + \sum_{k=1}^{N-1}
     &D_2 {L}_d(x_{k-1},  x_k)[\delta x_k] + D_1 {L}_d(x_k, x_{k+1})[\delta x_k ].
\end{aligned}
\end{equation}
where $D_i$ denotes the derivative with respect to the $i$-{th} argument.  
By the least action principle, the stationary point can be determined by letting the derivative of $\delta x_k$ be zero:
\begin{equation}
\label{eq:discrete_eom}
D_1L_d(x_k, x_{k+1}) + D_2L_d(x_{k-1}, x_{k}) = 0. 
\end{equation}


To incorporate the external force, we regard $u(t)$ as a covector,
$u(t)\in \mathrm{T}_{x(t)}^*\mathcal{M}$, acting on the virtual displacement
$\delta x(t)\in \mathrm{T}_{x(t)}\mathcal{M}$. The virtual work contribution is
approximated by the trapezoidal rule:
\begin{equation}
\begin{aligned}
        &\quad \int_{t_k}^{t_{k+1}} u(t)(\delta x(t))\, dt \\
        &\approx 
        \frac{\Delta t}{2} u_k[\delta x_k]
        + \frac{\Delta t}{2} u_{k+1}[\delta x_{k+1}].
\end{aligned}
\end{equation}

\subsection{Discrete-time Dynamics of Rigid Bodies }
Now, we derive the EoM on $\mathrm{SO}(3)\times \mathbb{R}^3$ using the Lie Group Variational Integrator (LGVI). Consider the discrete equation of motion by one-step Euler integration:
\begin{equation}
\label{eq:SO3R3-kin}
\begin{aligned}
    R_{k+1} &= R_kF_k \in \mathrm{SO}(3), \\
    p_{k+1} &= p_k + v_k \Delta t,
\end{aligned}
\end{equation}
with $R_k$ the orientation, $F_k$ the discrete rotation change, $p_k$ the position, and $v_k$ the linear velocity. The mid-point approximation can be applied:
\begin{equation}
\label{eq:Fk_midpoint}
\begin{aligned}
    F_k := R_k^{-1}R_{k+1} \approx I + \Delta t\omega_k^{\wedge}, \ \ \omega_k^{\wedge} \approx \frac{F_k - I}{\Delta t},
\end{aligned}
\end{equation}
\begin{equation}
\label{eq:pk_midpoint}
\begin{aligned}
    \dot{p}_k = v_k \approx \frac{p_{k+1} - p_k}{\Delta t}.
\end{aligned}
\end{equation}
The kinetic and potential energy can be approximated by:
\begin{align}
    \nonumber T_d&:=\frac{1}{2\Delta t}\operatorname{tr}((F_k - I)I^b(F_k - I)^{\intercal}) + \\ 
    & \quad \quad \frac{1}{2\Delta t}m\|p_{k+1} - p_k\|^2, \\
    V_d&:= m\left(\frac{p_{k+1}+p_{k}}{2}\right)^{\intercal} g \Delta t,
\end{align} 
where $I^b$ is the nonstandard moment of inertia \citep{marsden1999discrete} that relate the standard moment of inertia $I_b$ by $I_b = \operatorname{tr}(I^b)I_3-I^b$. Taking variation on $\mathrm{SO}(3) \times \mathbb{R}^3$, we have the dynamics:
\begin{align}
\label{eq:lgvi_dynamics}
    &F_{k+1}I^b - I^bF_{k+1}^{\intercal}=I^bF_{k} - F_k^{\intercal}I^b, \\ 
    &mv_{k+1} = mv_{k} + mg\Delta t. 
\end{align}
With the constraints formulated on $\mathrm{SO}(3)\times \mathbb{R}^3$, we can obtain the dynamics for multi-body systems. Compared to an explicit integration scheme, the LGVI naturally obeys the manifold constraints and conserves the energy \citep{marsden2001discrete, lee2005lie}. As the LGVI is completely in matrix form, there is no need to move back and forth between the Lie group and its Lie algebra for integration. The reader can refer to \citep{teng2024convex} for a detailed comparison of integrators on Lie groups for motion planning. The dynamics on $\mathrm{SO}(3)\times \mathbb{R}^3$ and $\mathrm{SE}(3)$ is summarized in~\Cref{table:diff-rotation}.  

\section{Differentiate the Rigid Body Dynamics}
\label{sec:diff-srb-dynamics} 
In this section, we derive the first- and second-order derivatives for the rigid body motions. By~\Cref{def:2nd-lie}, we apply the BCH formula of the Lie exponential map to obtain the second-order model of the rigid body dynamics on matrix Lie groups. 

\subsection{Retraction on $\mathrm{SE}(3)$ and $\mathrm{SO}(3)\times\mathbb{R}^3$}
Consider $X \in \mathcal{G}$, the second-order retraction can be obtained by the Taylor expansion of the exponential map:
\begin{equation}
\begin{aligned}
        c(t)=X\exp(t\xi)
            \approx X(I+t\xi^\wedge+\frac{(t\xi^{\wedge})^2}{2})
\end{aligned}
\end{equation}
For $\mathcal{G}=\mathrm{SE}(3)$, we have the second-order retraction as
\begin{equation}
    c_{X}(t)\approx \begin{bmatrix}
        R & p \\
        0 & 1
    \end{bmatrix} \begin{bmatrix}
        I+t\xi_R^\wedge + \frac{t^2}{2}\xi_R^{\wedge2} & t\xi_p + \frac{t^2}{2}\xi_R^\wedge \xi_p \\
        0 & 1
    \end{bmatrix}.
\end{equation}
For $\mathcal{G}=\mathrm{SO}(3)\times \mathbb{R}^3$, similarly we have:
\begin{equation}
    c_{(R, p)}(t) \approx \left(R(I+t\xi_R^\wedge + \frac{t^2}{2}\xi_R^{\wedge2}), p + t\xi_p\right),
\end{equation}
where the perturbation of position is not coupled with the change of orientation. 
\subsection{Differentiate the Kinematics}
To differentiate the kinematics of rotation in~\eqref{eq:SO3R3-kin}, we consider the kinematic chain constraints on $\mathcal{G}$ with finite length $n$. Without loss of generality, assume we have:
\begin{equation}
\label{eq:cons_kin}
    X_1X_2\cdots X_{n-1}X_n = I \in \mathcal{G}.
\end{equation}
Now we leverage the BCH formula to derive the second-order model at the operating point $\overline{X}_1, \overline{X}_2, \cdots, \overline{X}_n$. We denote $\overline{Y}:=\overline{X}_1\overline{X}_2 \cdots \overline{X}_N$. To avoid the use of Jacobians of $\log(\cdot)$ evaluated at point other than $I$, we reformulated constraints \eqref{eq:cons_kin} by multiplying $\overline{Y}^{-1}$ on both sides:
\begin{equation}
\begin{aligned}
&\overline{Y}^{-1}X_1X_2\cdots X_{n-1}X_n = \overline{Y}^{-1},
\end{aligned}
\end{equation}
and vectorize it via the logarithmic map:
\begin{equation}
\label{eq:vec-kin}
    \log{(\overline{Y}^{-1}X_1\cdots X_{n-1}X_n)} = \log{(\overline{Y}^{-1})} \in \mathbb{R}^{\dim \mathfrak{g}}.
\end{equation}
Denote $X_{i, j} = X_i X_{i+1} \cdots X_j$ when $i \le j$ and $X_{i+1, i} = I$, we consider the following identity to move the perturbation from the LHS to the RHS by the adjoint operation:
\begin{equation}
\label{eq:adj-id}
    \begin{aligned}
    \exp{(t\xi)\overline{X}_{i,j}} 
    &= \overline{X}_{i,j} (\overline{X}_{i,j}^{-1} \exp{(t\xi)} \overline{X}_{i,j}) \\ 
    &= \overline{X}_{i,j} \exp{(t \operatorname{Ad}_{\overline{X}_{i, j}^{-1} } \xi)}. 
\end{aligned}
\end{equation}
Then we apply \eqref{eq:adj-id} to group $\xi_i, \in 1,2,\dots, n$ in the kinematic chain:
\begin{equation}
\begin{aligned}
&\overline{Y}^{-1}\overline{X}_{1}\exp{(t\xi_1)}\cdots\overline{X}_{n-1}\exp{(t\xi_{n-1})} \overline{X}_{n}\exp{(t\xi_n)}\\
    =&\overline{Y}^{-1}\overline{X}_{1}\exp{(t\xi_1)} \cdots \\
    &\overline{X}_{n-1} \overline{X}_n \exp{(t \operatorname{Ad}_{\overline{X}_{n, n}^{-1} } \xi_{n-1})}\exp{(t\xi_n)}\\
     =&\overline{Y}^{-1}\overline{X}_{1}\exp{(t\xi_1)} \cdots  \overline{X}_{n-2} \\
     & \overline{X}_{n-1} \overline{X}_n  (\overline{X}_{n}^{-1} \overline{X}_{n-1}^{-1} \exp{(t\xi_{n-2})} \overline{X}_{n-1} \overline{X}_n) \\
     & \exp{(t \operatorname{Ad}_{\overline{X}_{n, n}^{-1} } \xi_{n-1})}\exp{(t\xi_n)}\\
     =&\overline{Y}^{-1}\overline{X}_{1}\exp{(t\xi_1)} \cdots  \overline{X}_{n-2}\overline{X}_{n-1} \overline{X}_n  \cdots\\
     & \exp{(t \operatorname{Ad}_{\overline{X}_{n-1, n}^{-1} }\xi_{n-2})}\exp{(t \operatorname{Ad}_{\overline{X}_{n, n}^{-1} } \xi_{n-1})}\exp{(t\xi_n)}\\
    =&\overline{Y}^{-1}\overline{X}_{1,n}\exp{(t \operatorname{Ad}_{\overline{X}_{2, n}^{-1} } \xi_{1})}\exp{(t \operatorname{Ad}_{\overline{X}_{3, n}^{-1} } \xi_{2})}  \exp{(t\xi_n)} \\
    =&\exp{(t \operatorname{Ad}_{\overline{X}_{2, n}^{-1} } \xi_{1})}\exp{(t \operatorname{Ad}_{\overline{X}_{3, n}^{-1} } \xi_{2})} \cdots \exp{(t\xi_n)}.
\end{aligned}
\end{equation}
Then we proceed to apply the BCH formula:
\begin{equation}
\begin{aligned}
     &\log{\left( \prod_{i=1}^n \exp{\left( t\operatorname{Ad}_{\overline{X}^{-1}_{i+1, n}} \xi_i \right)} \right)} \\
    =& t\operatorname{Ad}_{\overline{X}^{-1}_{2, n}} \xi_1 + \log{\left( \prod_{i=2}^n \exp{\left( t\operatorname{Ad}_{\overline{X}^{-1}_{i+1, n}} \xi_i \right)} \right)} \\
    +& \frac{1}{2} \left[t\operatorname{Ad}_{\overline{X}^{-1}_{2, n}} \xi_1, \log{\left( \prod_{i=2}^n \exp{\left( t\operatorname{Ad}_{\overline{X}^{-1}_{i+1, n}} \xi_i \right)} \right)}\right] \\
    &\quad + \mathcal{O}(t^3).
\end{aligned}
\end{equation}

By repeated application of BCH formula to $\log{\left( \prod_{i=k}^n \exp{\left( t\operatorname{Ad}_{\overline{X}^{-1}_{i+1, n}} \xi_i \right)} \right)}, k\ge 2$, we have the second-order model as: 
\begin{equation}
     \begin{aligned}
         & \quad \quad \log \left( \prod_{i=1}^n \exp{\left( t\operatorname{Ad}_{\overline{X}^{-1}_{i+1, n}} \xi_i \right)} \right) = \\
          &t\sum_{i=1}^n \operatorname{Ad}_{\overline X_{i+1, n}^{-1}} \xi_i + \sum_{i < j} \frac{t^2}{2}[\operatorname{Ad}_{\overline X_{i+1, n}^{-1}} \xi_i, \operatorname{Ad}_{\overline X_{j+1, n}^{-1}} \xi_j] \\
          & \quad + \mathcal{O}(t^3)
     \end{aligned}
\end{equation}

\begin{remark}
If one directly linearizes~\eqref{eq:cons_kin} at an operating point $\overline{Y}\neq I$ without the shifting strategy, then the first-order expansion of
$$\log(X_1X_2\cdots X_n)$$necessarily involves the Jacobians of the logarithm map at poses other than the identity~\citep{barfoot2024state}.  This introduces a dependence on a particular logarithmic chart, since the logarithm map is not a globally smooth single-valued parameterization of $\mathcal{G}$. 
\end{remark}
\begin{remark}
The shifting strategy avoids this dependence by recentering the constraint at the identity before linearization. The resulting second-order model is expressed directly in the Lie algebra, which preserves the symmetry of the kinematic chain and avoids differentiating the log map.
\end{remark}

\begin{table*}[t]
\centering
\caption{Summary of equations of motion for discrete rigid body dynamics and their second-order expansions. We denote the first-order perturbation of position, rotation, velocity, and the discrete angular velocity as $\xi^p, \xi^R, \xi^v$ and $\xi^F$, respectively. All the derived equations can be further evaluated by the automatic differentiation method for implementations.}
\small
\label{table:diff-rotation}
{\renewcommand{\arraystretch}{1.02}
\setlength{\tabcolsep}{4pt}
\begin{tabular}{>{\centering\arraybackslash}m{1.55cm}
                >{\centering\arraybackslash}m{4.95cm}
                >{\centering\arraybackslash}m{9.3cm}}
\hline
 & Constraints & Second-order Expansions \\ \hline

\(\begin{array}{c}
\mathrm{SO}(3)\\
\text{Kin.}
\end{array}\)
&
\[
\begin{aligned}
\log\!\left(R_{k+1}^{-1}R_kF_k\right)=0 \in \mathfrak{so}(3) \\
Y:=R_{k+1}^{-1}R_kF_k
\end{aligned}
\]
&
\[
\begin{aligned}
&-Y^{-1}\xi^R_{k+1}+F_k^{-1}\xi^R_k+\xi_k^F \\
&+\frac12[-Y^{-1}\xi^R_{k+1},\,F_k^{-1}\xi_k^R]
+\frac12[-Y^{-1}\xi^R_{k+1},\,\xi_k^F]  +\frac12[F_k^{-1}\xi_k^R,\,\xi_k^F]
\end{aligned}
\]
\\ \hline

\(\begin{array}{c}
\mathbb{R}^3\\
\text{Kin.}
\end{array}\)
&
\[
p_{k+1}-p_k-v_k\Delta t=0 \in \mathbb{R}^3
\]
&
\[
\xi_{k+1}^p-\xi_k^p-\xi_k^v\Delta t
\]
\\ \hline

\(\begin{array}{c}
\mathrm{SO}(3)\\
\text{Dyn.}
\end{array}\)
&
\[
\begin{aligned}
    (F_{k+1}I^b-I^bF^{\top}_{k+1})^\vee
- \\
(I^bF_k-F_k^{\top}I^b)^\vee=0 \in \mathbb{R}^3 
\end{aligned}
\]
&
\[
\begin{aligned}
&\Bigl(F_{k+1}(\xi^F_{k+1})^\wedge I^b
+I^b(\xi^F_{k+1})^\wedge F_{k+1}^{\top}\Bigr)^\vee\\
&\quad-\Bigl(F_k(\xi_k^F)^\wedge I^b
+I^b(\xi_k^F)^\wedge F_k^{\top}\Bigr)^\vee\\
&\quad+\frac12\Bigl(
F_{k+1}(\xi^F_{k+1})^{\wedge 2}I^b
-I^b(\xi^F_{k+1})^{\wedge 2}F_{k+1}^{\top}
\Bigr)^\vee\\
&\quad-\frac12\Bigl(
F_k(\xi_k^F)^{\wedge 2}I^b
-I^b(\xi_k^F)^{\wedge 2}F_k^{\top}
\Bigr)^\vee
\end{aligned}
\]
\\ \hline

\(\begin{array}{c}
\mathbb{R}^3\\
\text{Dyn.}
\end{array}\)
&
\[
mv_{k+1}-mv_k-mg\Delta t=0 \in \mathbb{R}^3
\]
&
\[
m\xi_{k+1}^v-m\xi_k^v
\]
\\ \hline



\(\begin{array}{c}
\mathrm{SE}(3)\\
\text{Kin.}
\end{array}\)
&
\[
\left\{
\begin{aligned}
&\log\left(R_{k+1}^{-1}R_kF_k\right) = 0,\\
&p_{k+1} - p_k - R_kv_k\Delta t= 0
\end{aligned}
\right.
\]
&
\[
\begin{aligned}
&\mathrm{(rotation)} \text{  Identical to the SO(3) case.} \\
&\mathrm{(translation)} \\
& \xi^p_{k+1}
-F_k^{-1}\xi^p_k
-\Delta t\,\xi^v_k
-\Delta t\,F_k^{-1}(\xi_k^R)^\wedge v_k\\
&\quad+\frac12(\xi^R_{k+1})^\wedge\xi^p_{k+1}
-\frac12F_k^{-1}(\xi_k^R)^\wedge\xi^p_k\\
&\quad-\Delta t\,F_k^{-1}(\xi_k^R)^\wedge F_k\xi^v_k
-\frac{\Delta t}{2}(\xi_k^F)^\wedge\xi_k^v \\
&-\frac{\Delta t}{2}F_k^{-1}(\xi_k^R)^{\wedge 2}v_k
\end{aligned}
\]\\ \hline

\(\begin{array}{c}
\mathrm{SE}(3)\\
\text{Dyn.}
\end{array}\)
&
\[
\left\{
\begin{aligned}
&F_{k+1}I^b-I^bF_{k+1}^{\top} -(I^bF_k-F_k^{\top}I^b)=0,\\
&mv_{k+1}-(mF_k^{\top}v_k+mR_{k+1}^{\top}g\Delta t)=0
\end{aligned}
\right.
\]
&
\[
\begin{aligned}
&\mathrm{(rotation)} \text{  Identical to the SO(3) case.} \\
&\mathrm{(translation)}\ mF_{k+1}\xi^v_{k+1}
-m\xi^v_k
+m(\xi_k^F)^\wedge F_k^{\top}v_k \\
& +m\Delta t\,(\xi^R_{k+1})^\wedge R_{k+1}^{\top}g\\
&\quad+\frac m2F_{k+1}(\xi^F_{k+1})^\wedge\xi^v_{k+1}
+\frac m2(\xi_k^F)^\wedge\xi_k^v\\
&\quad-\frac m2(\xi_k^F)^{\wedge 2}F_k^{\top}v_k
-\frac{m\Delta t}{2}(\xi^R_{k+1})^{\wedge 2}R_{k+1}^{\top}g
\end{aligned}
\]\\ \hline
\end{tabular}}
\end{table*}

\subsection{Differentiate the Dynamics}
\begin{table*}[t]
\centering
\caption{Comparison of rotation dynamics and its possible singularities. We note that our derivation in~\Cref{table:diff-rotation} is free of any of the problems. For the computation of quaternions, we have $q=\left[\begin{array}{l}
q_w \\
q_v
\end{array}\right],\ q_w \in \mathbb{R},\ q_v \in \mathbb{R}^3$ and $L(q)=\left[\begin{array}{cc}
q_w & -q_v^{\top} \\
q_v & q_w I_3+q_v^{\wedge}
\end{array}\right] \in \mathbb{R}^{4 \times 4}$ }
\label{table:rotation-coordinate-comparison}
\small
{\renewcommand{\arraystretch}{1.05}
\setlength{\tabcolsep}{4pt}
\begin{tabular}{>{\centering\arraybackslash}m{4.2cm}
                >{\centering\arraybackslash}m{2.2cm}
                >{\centering\arraybackslash}m{7.0cm}
                >{\centering\arraybackslash}m{2.4cm}}
\hline
Coordinate & Integration Scheme & Formula & Singularity \\ 
\hline

Rotation matrix $R\in \mathbb{R}^{3\times 3}, \omega \in \mathbb{R}^3$ & Explicit Euler &
\[
\begin{aligned}
R_{k+1} &= R_k\exp(\omega_k\Delta t)\\
\omega_{k+1} &= \omega_k+I^{-1}((I\omega_k)\times\omega_k)\Delta t
\end{aligned}
\]
& $\exp(\cdot)$ lose rank at $\|\omega_k\Delta t\|=\pi$ \\[1mm] \hline

Euler Angle$\ \ \theta, \dot{\theta} \in \mathbb{R}^{3}$ & Explicit Euler &
\[
\begin{aligned}
\theta_{k+1} &= \theta_k+J(\theta_k)\dot{\theta}_k\Delta t\\
\dot{\theta}_{k+1} &= \dot{\theta}_k
+J(\theta_k)^{-1}I^{-1}((I\omega_k)\times\omega_k)\Delta t
\end{aligned}
\]
& Gimbal Lock. The Jacobian $J(\theta)$ loses rank. \\[1mm] \hline

Quaternion $\ \ q \in \mathbb{R}^4, \omega \in \mathbb{R}^3 $ & Variational &
\[
\begin{aligned}
q_{k+1} &= L(q_k)\bar{\omega}_k,\quad M_{k+1} + m_{k+1}= M_k - m_k \\
M_k&:=I\omega_k\sqrt{1-\|\omega_k\|^2}, m_k=\omega_k\times(I\omega_k) \\
\bar{\omega}&:=[\sqrt{1-\|\omega_k\|^2},\omega^\top]^\top
\end{aligned}
\]
& The square root \\[1mm] \hline

Quaternion $\ \ q \in \mathbb{R}^4, \omega \in \mathbb{R}^3 $  & Explicit Euler &
\[
\begin{aligned}
q_{k+1} &= q_k+\frac{1}{2}L(q_k)\bar{\omega}_k\Delta t\\
\omega_{k+1} &= \omega_k+I^{-1}((I\omega_k)\times\omega_k)\Delta t \\
\bar{\omega}&:=[0,\omega^\top]^\top
\end{aligned}
\]
& No manifold constraints \\[1mm] \hline

\end{tabular}}
\end{table*}
Consider the rotational dynamics in \eqref{eq:lgvi_dynamics}, we again use the retraction curve along the direction $\xi^{F\wedge}_i \in \mathfrak{so}(3)$:
\begin{equation}
    c(t) = \overline{F}_i\exp{(t\xi^{F}_i)} \in \mathrm{SO}(3). 
\end{equation}
As \eqref{eq:lgvi_dynamics} is already a skew matrix, the vectorized perturbed dynamics can be obtained by substituting the retraction curve and applying the $(\cdot)^{\vee}$ map:
\begin{equation}
\begin{aligned}
    &(F_{k+1}\exp{(t\xi^{F}_{k+1})} I^b - I^b\exp{(-t\xi^{F}_{k+1})}F_{k+1}^{\top})^{\vee}= \\
    & \quad (I^bF_{k}\exp{(t\xi^{F}_{k})} - \exp{(-t\xi^{F}_{k})}F_k^{\top}I^b)^{\vee}. 
\end{aligned}
\end{equation}
The second-order expansion can be obtained by substituting the series and keeping the first and second-order terms of $t$.

For the dynamics on $\mathrm{SE}(3)$, we can follow the same procedure to derive the second-order model. Finally, we summarize all the constraints and the corresponding second-order expansions in \Cref{table:diff-rotation}.

\subsection{Singularity of Differentiation Schemes}
Now we compare the proposed differentiation scheme with the existing method in terms of the possible singularities in the rotational dynamics.

As shown in the first and third rows, the second-order model only includes the adjoint action (Lie bracket) for change of frame by matrix multiplication. Thus, the differential of the rotation dynamics and kinematics has no singularity, as the matrix representation is smooth everywhere. 

The possible alternatives are summarized in~\Cref{table:rotation-coordinate-comparison}. As indicated in~\citep{bonalli2019trajectory}, there is no need to impose the explicit manifold constraints in the ambient space if one can properly map the generalized velocity in the tangent space to the manifold. However, we note that on $\mathrm{SO}(3)$, such a mapping inevitably introduces singularities as it is not possible to embed $\mathbb{RP}^3$ in $\mathbb{R}^3$ as a limitation from the embedding theorem~\citep{hirsch2012differential}. As a consequence, any mapping $\mathbb{R}^3\rightarrow \mathbb{SO}(3)$ or $\mathbb{R}^3\rightarrow \mathbb{SU}(2)$ can not be globally differentiable. The possible examples are shown in the third column of~\Cref{table:rotation-coordinate-comparison}. To avoid such an issue, one can model the rotation in the ambient space with explicit manifold constraints, which has been shown in~\citep{teng2025optimization} to have poor convergence. As shown in the last row, one can also discard the manifold constraints with explicit constraints, while the state will no longer stay on the manifold. 
\section{Constrained Optimization on Manifolds} 
\label{sec:RIPM}
In this section, we introduce the \methodlong as the optimization backend for \ref{eq:TO}. A Riemannian Interior Point Method is implemented in \citep{lai2024riemannian} using Manopt~\citep{boumal2014manopt} to derive the gradients in MATLAB; however, it does not apply any advanced globalization strategy to enhance the robustness or convergence. In this work, we provide a customized C++ implementation without dependencies on Manopt. We refer to the classical Interior Point Method (IPM) implementation, IPOPT \citep{wachter2006implementation}, for the backtracking line search and inertia correction to improve robustness. 



\subsection{Optimality Conditions on Manifolds}

We introduce the optimality condition of constrained optimization defined on smooth manifolds. By the language of differential geometry, the condition coincides with its counterparts in Euclidean space \citep{yang2014optimality, absil2008optimization}.  
Consider an optimization problem with general nonlinear constraints defined on a smooth manifold $\mathcal{M}$:
\begin{problem}[Constrained Optimization on Manifolds]
\label{prob:cons-nlp}
\begin{equation}
\begin{aligned}
    \min_{x \in \mathcal{M} } \quad &f(x) \\
    \mathrm{s.t.}\quad &h_i(x) = 0,\ i = 1, 2, \cdots, l, \\
    &g_j(x) \le 0,\ j = 1, 2, \cdots, m.
\end{aligned}    
\end{equation}
with the variable $x$ defined on a smooth manifold $\mathcal{M}$, $f(x)$ the cost function, $\{h_i(x)\}_{i=1}^{l}$ the equality constraints and $\{g_i(x)\}_{i=1}^m$ the inequality constraints. 
\end{problem}

With the multiplier $y \in \mathbb{R}^{l}$ and $z \in \mathbb{R}^m$, we have the Lagrangian defined on $\mathcal{M}\times \mathbb{R}^l \times \mathbb{R}^m$:
\begin{equation}
    \mathcal{L}(x, y, z) := f(x) + \sum_{i=1}^{l} y_i h_i(x) + \sum_{i=1}^{m} z_i g_i(x).
\end{equation}
The differential of the Lagrangian with respect to $x$ is
\begin{equation}
\begin{aligned}
    \mathrm{D}_x\mathcal{L}(x, y, z) 
    = &\mathrm{D} f(x) + \sum_{i=1}^{l} y_i \mathrm{D} h_i(x)\\
    + &\sum_{i=1}^{m} z_i \mathrm{D} g_i(x),
\end{aligned}
\end{equation}
which is a covector in $\mathrm{T}_x^*\mathcal{M}$. Then we have the following optimality conditions:
\begin{definition}[(First-Order Optimality Conditions)]
\label{theorem:kkt}
    $x^{\star} \in \mathcal{M} $ is said to satisfy the KKT conditions if there exist multipliers $y^{\star}$ and $z^{\star}$ such that:
    \begin{enumerate}
        \item Stationarity condition: $\mathrm{D}_x\mathcal{L}(x^{\star}, y^{\star}, z^{\star}) = 0$,
        \item Primal feasibility: $h_i(x^{\star}) = 0, \; g_j(x^{\star}) \le 0, \forall i, j$,
        \item Dual feasibility: $z^{\star}_j \ge 0, \forall j$,
        \item Complementarity condition: $z^{\star}_j g_j(x^{\star}) = 0, \forall j$.
    \end{enumerate}
\end{definition}We then proceed to introduce the \method to search for the critical point that satisfies the first-order optimality condition.

\subsection{\methodlong}
\begin{algorithm}[t]
\caption{\methodlong (\method)}\label{alg:line-search-RIPM}
\begin{algorithmic}[1]
\footnotesize
\Require Initial guess $(x, y, z, s) \in \mathcal{M} \times \mathbb{R}^l \times \mathbb{R}^m \times \mathbb{R}^m$, homotopy parameter $\mu$, parameters in \Cref{table:param-RIPM}

\For{$k = 1, \ldots, N_{\max}$}
    \State {\texttt{// Check the termination condition}}
    \If {$E_0(x, y, z, s) \le \epsilon_{\mathrm{tol}}$}
        \State \textbf{break} 
    \EndIf
    \State
    \State {\texttt{// Decide the search direction}}
    \State $(J_{\zeta_\mu}, \zeta_\mu) \gets$ \eqref{eq:kkt-vec} and \eqref{eq:kkt-jac} evaluated at $(x, y, z, s)$
        \State $(J_{\zeta_\mu}, \delta_w^{\mathrm{last}}) \gets (J_{\zeta_\mu}, \mu, \delta_w^{\mathrm{last}})$ \Comment{\Cref{alg:inertia-correction}}
    \State $(\xi_x, \xi_{y}, \xi_z, \xi_s) \gets (J_{\zeta_\mu}, \zeta_\mu)$   \Comment{Newton step \eqref{eq:newton-kkt}}

    \State
    \State {\texttt{// Update the homotopy parameter}}     
    \If {$E_\mu(x, {y}, z, s) \le 10\mu$}
        \State $\mu \gets \max\{ \frac{\epsilon_{\mathrm{tol}}}{10}, \min \{ \kappa_{\mu}\mu, \mu^{\theta_{\mu}} \} \}$
    \EndIf

    \State
    \State {\texttt{// Step size by fraction to boundary rule}}
    \State $\tau \gets \max \left\{\tau_{\min}, 1-\mu\right\}$
    \State $\alpha^z \gets \max\{ \alpha \in (0, 1] \ | \ z + \alpha \xi_z \ge (1 - \tau_j) z \} $
    \State $\alpha^s \gets \max\{ \alpha \in (0, 1] \ | \ s + \alpha \xi_s \ge (1 - \tau_j) s \} $

    \State $\xi_{z} \gets \alpha^{z} \xi_{z} $
    \State $\xi_{s} \gets \alpha^{s} \xi_{s} $

    \State
    \State {\texttt{// Backtracking line-search}}
    \State $ \varphi_{\mu, k}^* \gets \varphi_{\mu}(x) $
    \State $\theta^* \gets \|h(x)\|_1 + \sum_{i, g_i(x) \ge 0} g_i(x) $
    
    \For {$j = 1, \ldots, j_{\max}$}
        \State $\Tilde{x} \gets \operatorname{R}_{x}(\xi_x)$ \Comment{Retraction map}

        \State $\Tilde{\varphi_{\mu}} \gets \varphi_{\mu}(\Tilde{x}) $, \Comment{Loss at trial point}
        \State $\Tilde{\theta} \gets \|h(\Tilde{x})\|$ \Comment{Constraints violation}
        \State $c \gets \operatorname{D}_{x} f(\Tilde{x})[\xi_x]$ \Comment{Improvement of loss}

        \\
        \If {$\theta^* \le \theta_{\min}$ \textbf{and} $c<0$ \textbf{and} $\alpha_{k, j}c^{s_{\varphi}} > \delta \theta^{*s_{\theta}} $}
        \If { $\varphi_{\mu}(\Tilde{x}) \le \varphi_{\mu}(x) + \eta_{\varphi}c $ }
        \State \textbf{break} \Comment{Armijo condition satisfied}
        \EndIf
        
        \ElsIf {$\Tilde{\varphi}_{\mu} \le {\varphi}_{\mu, k}^* - \gamma_{\theta} {\theta}^* $ \textbf{or} {$\Tilde{\theta} \le (1 - \gamma_{\theta}) {\theta}^* $}}
        \State \textbf{break} \Comment{Feasibility or cost improved}
        \EndIf
        
        \State $\xi_x \gets \beta \xi_x$ \Comment{Reduce the step size}
    \EndFor
    \State
    \State {\texttt{// Update the variables}}
    \State $x \gets \Tilde{x}$
    \State $y \gets y + \xi_y$
    \State $z \gets z + \xi_z$
    \State $s \gets s + \xi_s$
\EndFor \\
\Return $(x, y, z, s)$
\end{algorithmic}
\end{algorithm}

We apply a line-search scheme with inertia correction in the \method framework to find the KKT pair $(x^{\star}, y^{\star}, z^{\star}) \in \mathcal{M}\times \mathbb{R}^l \times \mathbb{R}^m$ that satisfies the optimality condition in \Cref{theorem:kkt}. As an implementation of IPM, we consider the log-barrier problem with homotopy parameter $\mu$ and slack variables $s$:
\begin{problem}[Log-barrier Problem]
\begin{equation}
\begin{aligned}
    \min_{x \in \mathcal{M}, s} \quad \varphi_{\mu} := &f(x) - \mu \sum_{i=1}^{m} \log{s_i} \\
    &h_i(x) = 0, i = 1, 2, \cdots, l, \\
    &g_j(x) + s_j = 0, j = 1, 2, \cdots, m, \\
    & s \ge 0.
\end{aligned}
\end{equation}
\end{problem}
{\color{black}
By letting $\mu \rightarrow 0$, $\varphi_\mu$ converges to $\varphi_0$~\citep{wright2006numerical}. To search the local optimum for each $\varphi_\mu$, we derive the KKT vector field and search for the root that satisfies the optimality condition. Here we consider multipliers $y, z$ and define the variable $w = (x,y,z,s) \in \mathcal{K} :=\mathcal{M}\times \mathbb{R}^l \times \mathbb{R}^m\times \mathbb{R}^m$. The KKT vector field takes the form:
}
\begin{equation}
\label{eq:kkt-vec}
\begin{aligned}
    \zeta_\mu &: \mathcal{K} \rightarrow \operatorname{T}^*\mathcal{K} \\
    \zeta_{\mu}(w) &:= \begin{bmatrix}
\operatorname{D}_x f(x) + A_E(x)y + A_I(x)z \\
h(x) \\
g(x)+s \\
Sz - \mu e
\end{bmatrix},
\end{aligned}
\end{equation}
with the Jacobian of the constraints:
\begin{equation}
    A_{\mathrm{E}}(x): \mathbb{R}^{l} \rightarrow \operatorname{T}^*_x\mathcal{M}:  A_{\mathrm{E}}(x)y := \sum_{i=1}^{l} y_i \operatorname{D}_x h_i(x),
\end{equation}
\begin{equation}
    A_{\mathrm{I}}(x): \mathbb{R}^{m} \rightarrow \operatorname{T}^*_x\mathcal{M}: A_{\mathrm{I}}(x)z := \sum_{i=1}^{m} z_i \operatorname{D}_x g_i(x).
\end{equation}
Then we have the Newton iteration as:
\begin{equation}
\label{eq:newton-kkt}
    J_{\zeta_\mu}(w) \xi_w = - \zeta_\mu(w)
\end{equation}
with the Jacobians of the KKT vector field as:
\begin{equation}
\label{eq:kkt-jac}
\begin{aligned}
J_{\zeta_\mu} &: \operatorname{T}\mathcal{K}\rightarrow \operatorname{T}^*\mathcal{K} \\
J_{\zeta_\mu} &= \left[\begin{array}{cccc}
H(x, y, z) & A_{\mathrm{E}}(x) & A_{\mathrm{I}}(x)  & 0 \\
A_{\mathrm{E}}^*(x) & 0 & 0 & 0 \\
A_{\mathrm{I}}^*(x) & 0 & 0 & I \\
0 & 0 & S & Z
\end{array}\right],
\end{aligned}
\end{equation}
$(\xi_x, \xi_y, \xi_z, \xi_s) \in \operatorname{T}\mathcal{K} $ the search direction, $Z = \operatorname{diag}(z)$,  $S = \operatorname{diag}(s)$ and {$A^{*}$ the adjoint of $A$}. The Hessian of the Lagrangian can be obtained by:
\begin{equation}
\begin{aligned}
    H(x, y, z) = &\operatorname{Hess}_x f(x) + \sum_{i=1}^{l} y_i \operatorname{Hess}_x h_i(x)\\ + &\sum_{i=1}^{m} z_i \operatorname{Hess}_x g_i(x).
\end{aligned}
\end{equation}

After obtaining the search direction via \eqref{eq:newton-kkt}, we implement a backtracking line-search method with inertia correction to decide the step size \citep{wachter2006implementation}. The rationale of inertia correction is to add a block diagonal regularizer to the KKT systems to ensure the number of positive and negative eigenvalues is correct. The procedure of the inertia correction is summarized in \Cref{alg:inertia-correction} in the Appendix~\ref{sec:inertia}. To fully utilize the symmetry of the KKT systems, we consider reducing the KKT matrix $J_{\zeta_\mu}$ to a symmetric indefinite matrix as summarized in the Appendix~\ref{appx:sym-lin-sys}.  

The procedure of \method is summarized in \Cref{alg:line-search-RIPM}. After an initial step size is determined, \Cref{alg:line-search-RIPM} accepts the step only if the improvement in feasibility or the cost reduction is sufficient. The condition for adjusting the homotopy parameter for the log-barrier problem is determined by the violation of the optimality condition $E_{\mu}$:
\begin{equation}
\label{eq:converge}
E_{\mu} = \max\left\{ \epsilon_{\mathrm{KKT}}, \epsilon_{\mathrm{E}},
\epsilon_{\mathrm{I}}, \right\},
\end{equation}
with the violation of the KKT condition, equality constraints and inequality constraints being
\begin{equation}
\begin{aligned}
\epsilon_{\mathrm{KKT}} &= \frac{\| \operatorname{D}_xf(x)+A_{\mathrm{E}}y + {A}_Iz \|_{\infty}}{s_d}, \\
\epsilon_{\mathrm{E}} &= \| h(x) \|_{\infty}, \ \ \epsilon_{\mathrm{I}} = \frac{\| Sz - \mu e \|_{\infty}}{s_c},
\end{aligned}
\end{equation}
respectively. The $s_d$ and $s_c$ are normalizing parameters, and the termination condition for \Cref{prob:cons-nlp} is to check \eqref{eq:converge} when $\mu=0$, i.e., $E_{0}$. The meaning of all the parameters used in \method is presented in the Appendix. \ref{appx:ripm-table}.

\subsection{Convergence Analysis}

We briefly discuss the local and global convergence property of \Cref{alg:line-search-RIPM}. 

\textbf{Local Convergence}: We show the local convergence of \method under standard regularity condition assume the Linear Independence Constraint Qualification (LICQ), Strong Complementarity (SC), and Second-order Sufficient Condition (SOSC): 
\begin{restatable}[Local regularity at the KKT point]{assumption}{localregularity}
\label{ass:local-regularity}
Let $w^\star=(x^\star,y^\star,z^\star,s^\star)\in 
\mathcal{M}\times\mathbb{R}^l\times\mathbb{R}^m\times\mathbb{R}^m$
be a KKT point of \Cref{prob:cons-nlp}, with $s^\star=-g(x^\star)$.
Let $\mathcal{E}:=\{1,\dots,l\},\
    \mathcal{I}:=\{1,\dots,m\},\
    \mathcal{A}:=\mathcal{A}(x^\star):=\{j\in\mathcal{I}\mid g_j(x^\star)=0\}$ denotes the index set of the constraints. 
Assume that:
\begin{enumerate}
    \item LICQ: 
    the set $\{\mathrm{D}h_i(x^\star)\}_{i\in\mathcal{E}} \cup \{\mathrm{D}g_j(x^\star)\}_{j\in\mathcal{A}}$
    is linearly independent.
    \item SC:
    $z_j^\star>0$ for all $j\in\mathcal{A}$.
    \item SOSC:
    $H(x^\star,y^\star,z^\star)[\xi,\xi] > 0$
    for every nonzero $\xi \in \mathrm{T}_{x^\star}\mathcal{M}$ satisfying
    $\mathrm{D}h_i(x^\star)[\xi]=0,\ \forall i\in\mathcal{E},\quad
    \mathrm{D}g_j(x^\star)[\xi]=0,\ \forall j\in\mathcal{A}$.
\end{enumerate}
\end{restatable}

\begin{restatable}[Local superlinear convergence of \method]{theorem}{localconvergence}
\label{thm:local-superlinear}
Suppose \Cref{ass:local-regularity} holds at \(w^\star\). Then the primal-dual KKT Jacobian \(J_{\zeta_0}(w^\star)\) is nonsingular. Consequently, for each fixed \(\mu\) sufficiently small, the Newton iteration
\begin{equation}
J_{\zeta_\mu}(w_k)\,\xi_k=-\zeta_\mu(w_k),
w_{k+1}=\mathrm R_{w_k}(\xi_k),
\end{equation}
exhibits local superlinear convergence. Moreover, if $\mu_k=o(\|\zeta_0(w_k)\|)$,
then \(\{w_k\}\) converges superlinearly to \(w^\star\); if $\mu_k=\mathcal O(\|\zeta_0(w_k)\|^2),
$ then \(\{w_k\}\) converges quadratically to \(w^\star\).
\end{restatable}
The proof of the local convergence is inspired by~\citep{lai2024riemannian} and is deferred to Appendix~\ref {appx:proof-convergence}.


\textbf{Global Convergence}: 
\method uses backtracking line search and inertia correction as globalization mechanisms. The inertia correction enforces a solvable KKT system with correct curvature, while the line search ensures sufficient decrease. 
This globalization strategy is mainly used for robustness. We do not employ filter or restoration mechanisms as in~\citep{wachter2006implementation} and therefore do not claim full global convergence from arbitrary initialization. Despite the lack of global convergence, we show that the proposed method still outperforms the method with such mechanisms for~\eqref{eq:TO} on matrix Lie groups.  

\section{Experiments}
\label{sec:num-exp}
In this section, we provide a comprehensive evaluation of the \method on \ref{eq:TO} in numerical benchmarks and deploy \method for hardware experiments. 

\subsection{System setup}

\subsubsection{Cost function design}
We consider the quadratic cost function to indicate the difference between the desired and the actual configurations. For orientation and the discrete angular velocity, we consider the chordal distance-based cost~\citep{lee2010geometric}:
\begin{equation}
\label{eq:SO3-cost}
\begin{aligned}
        &\ell_{\mathrm{SO}(3)}(R, R_d,Q_R) \\
    = &\frac{1}{2}\operatorname{tr}\left( (RR^{\top}_d - I)Q_R(RR^{\top}_d - I)^{\top} \right)
\end{aligned}
\end{equation}
with $Q_R$ a positive-definite weighting matrix. For the position, we consider the quadratic cost function:
\begin{equation*}
    \ell_{\mathbb{R}^3}(p, p_d, Q_p) = \frac{1}{2}(p - p_d)^{\top}Q_p(p - p_d),
\end{equation*}
where $Q_p$ is a positive-definite weighting matrix. 

For \ref{eq:TO}, we have the state $x:=(R,p,F,v)$  in the configuration space $\mathcal{M}_{RB}$ as $\in \mathrm{SO}(3)\times\mathbb{R}^3\times\mathrm{SO}(3)\times \mathbb{R}^3$ or $\mathrm{SE}(3)\times\mathrm{SE(3)}$. The terminal cost then take the form:
\begin{equation*}
\begin{aligned}
        \ell_T(x_T) &= \ell_{\mathrm{SO}(3)}(R_T, R_{d,T}, P_R) + \ell_{\mathbb{R}^3}(p_T,p_{d,T}, P_p) \\
    & + \ell_{\mathrm{SO}(3)}(F_T, F_{d,T}, P_F) + \ell_{\mathbb{R}^3}(v_T,v_{d,T}, P_v).
\end{aligned}
\end{equation*}
and the running cost becomes
\begin{equation*}
\begin{aligned}
        \ell(x_k, u_k) &= \ell_{\mathrm{SO}(3)}(R_k, R_{d,k}, Q_R) + \ell_{\mathbb{R}^3}(p_k,p_{d,k}, Q_p) \\
    & + \ell_{\mathrm{SO}(3)}(F_k, F_{d,k}, Q_F) + \ell_{\mathbb{R}^3}(v_k,v_{d,k}, Q_v) \\
    & + \frac{1}{2}\|u_k-u_{d,k}\|^2_{Q_u}.
\end{aligned}
\end{equation*}
with $Q_{(\cdot)}$ and $P_{(\cdot)}$ the cost matrices.
For $\mathrm{SE}(3)$ and $\mathrm{SO}(3)\times \mathbb{R}^3$, we differentiate the cost function in the corresponding tangent spaces using the corresponding retraction curves. 

For quaternion $q = [q_w, q_{v}^\top]^\top$ with $q_w$ the scale part and $q_v$ the imaginary part, we define the attitude error by the relative quaternion as in~\citep{sun2022comparative}
\begin{equation}
\begin{aligned}
\delta q &=q_{\mathrm{ref}}^{-1}\otimes q \\
&=
\begin{bmatrix}
q_{\mathrm{ref},w}q_w+q_{\mathrm{ref},v}^{\top}q_v\\
q_{\mathrm{ref},w}q_v-q_wq_{\mathrm{ref},v}-q_{\mathrm{ref},v}\times q_v
\end{bmatrix}. 
\end{aligned}
\end{equation}
The quadratic cost penalizes only the imaginary part by
\begin{equation}
\label{eq:quat-cost}
 \ell_{\mathrm{SU}(2)}(q,q_{\mathrm{ref}})
=
\delta q^\top_vQ_q\delta q_v
\end{equation}
with $Q_q$ the positive definite weighting matrix. As only the imaginary part is penalized, \eqref{eq:quat-cost} admits two global minima when $q_w = 1$ or $-1$, which enable the solver to smoothly select the branch. 

In the experiments, we choose two proper isotropic weighting matrices $Q_q$ and $Q_R$ so that
$\ell_{\mathrm{SO}(3)}$ and $\ell_{\mathrm{SU}(2)}$ are exactly equivalent by the fact that
\begin{equation}
\label{eq:R=q}
    \|R-I\|_F^2 = 8\|q_v\|^2. 
\end{equation}
The detailed derivation is included in the Appendix. \ref{appx:q=R}.


\subsubsection{Software setup} We implemented 
\Cref{alg:line-search-RIPM} in C++. We apply the sparse linear solver MUMPS~\citep{amestoy2019performance} to solve the modified symmetric indefinite KKT systems as in Appendix~\Cref{appx:sym-lin-sys}. We use CasADi~\citep{andersson2019casadi} to evaluate the second-order models in \Cref{table:diff-rotation} and then generate the C++ code for function evaluations. 


\begin{figure}
    \centering
    \includegraphics[width=1\linewidth]{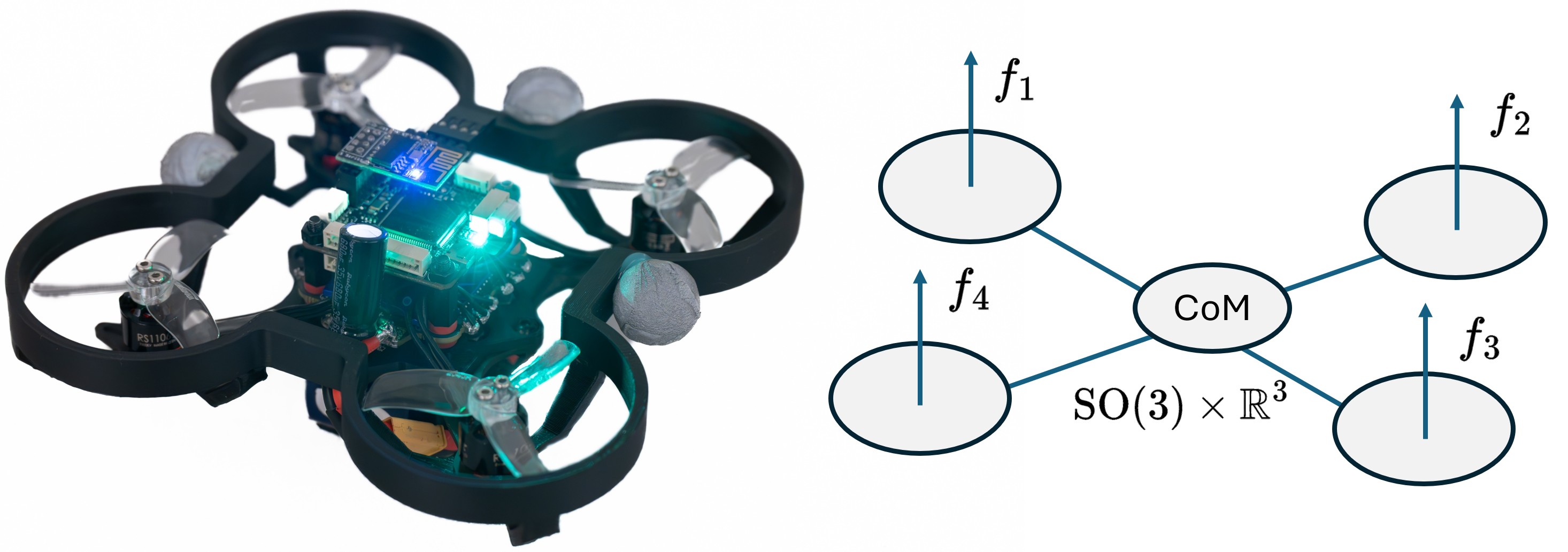}
    \caption{The hardware used in the experiment and the corresponding thrust force. We use a motion capture system for localization. }
    \label{fig:drone}
\end{figure}

\subsection{Convergence of \method}

\begin{figure*}[t]
    \centering
    \includegraphics[width=1\linewidth]{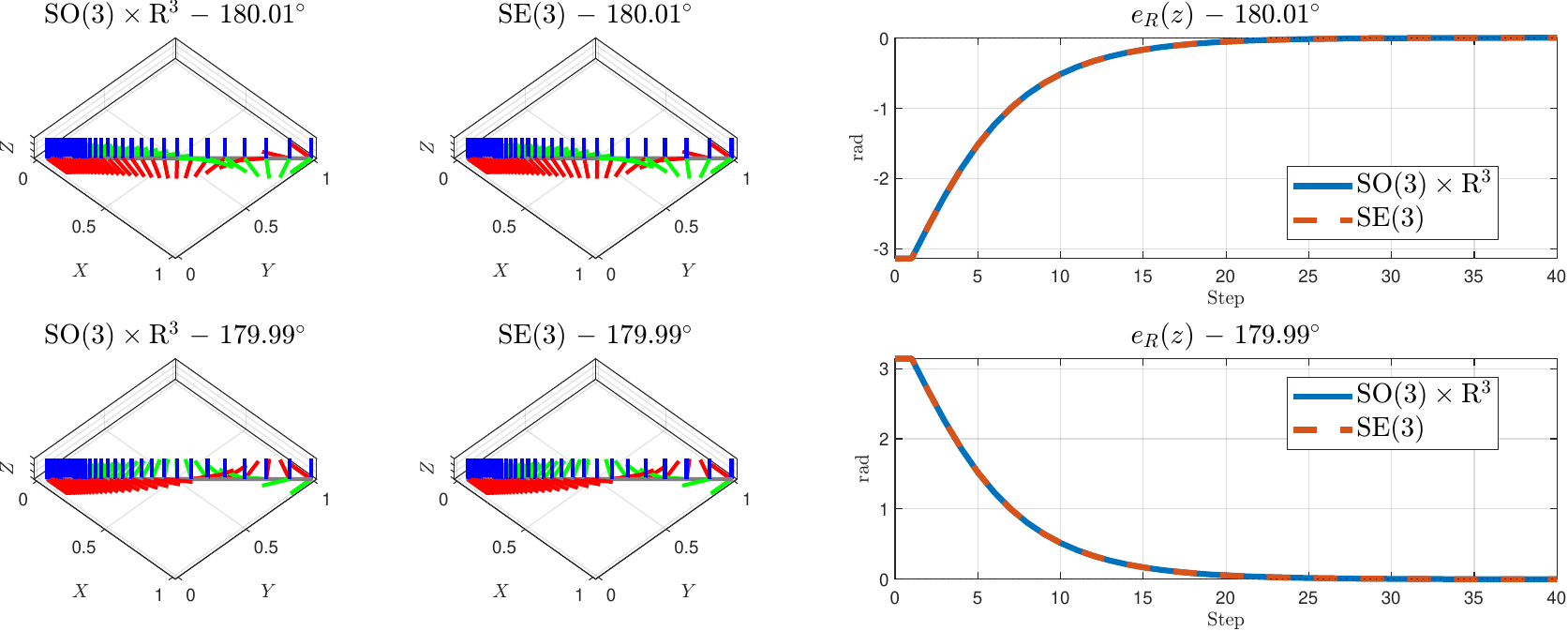}
    \caption{Comparison of the solutions of \method on $\mathrm{SE}(3)$ and $\mathrm{SO(3)}\times \mathbb{R}^3$ for steering a fully actuated rigid body with initial rotation around $z-$axis at location $[1,1,0]$ to the identity pose. We find that \method is capable of generating \emph{discontinuous} solutions as expected for rotation groups~\citep{kalabic2017mpc, bhat2000topological}.}
    \label{fig:SE3_SO3_traj}
\end{figure*}

\begin{figure}
    \centering
    \includegraphics[width=1\linewidth]{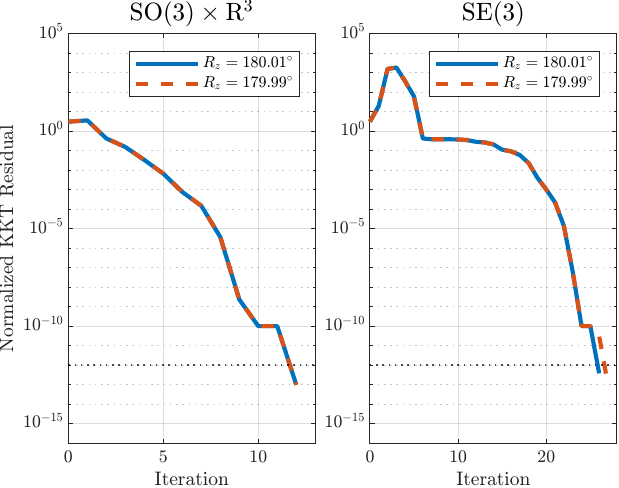}
    \caption{The convergence of the four cases in~\Cref{fig:SE3_SO3_traj}. We find that the proposed method exhibits superlinear convergence for both $\mathrm{SE(3)}$ and $\mathrm{SO}(3)\times \mathbb{R}^3$. This result matches the theory in~\citep{mahony2002geometry, absil2008optimization} and verifies that our analytical gradients in~\Cref{table:diff-rotation} are derived correctly.}
    \label{fig:SE3_SO3_converge}
\end{figure}

We first test a few numerical examples to show the convergence of the proposed method on $\mathrm{SO}(3)\times \mathbb{R}^3$ and $\mathrm{SE}(3)$. We show that the derived gradients based on the Lie exponential enable superlinear convergence.  

We consider steering a fully actuated rigid body with an initial rotation around the $z-$axis for $180^\circ\pm0.01^\circ$ that is perturbed around the singular pose $180^\circ$. In \Cref{fig:SE3_SO3_traj}, we see that the proposed method is capable of finding solutions on both $\mathrm{SE}(3)$ and $\mathrm{SO}(3)\times \mathbb{R}^3$. As implied by the topology of the rotation group~\citep{bhat2000topological, kalabic2017mpc}, a small perturbation around this singular pose yields a discontinuous change in the geodesics that results in a discontinuous control law. As in \Cref{fig:SE3_SO3_converge}, both methods exhibit superlinear convergence. Though the $\mathrm{SE}(3)$ case is not applying the Riemannian retraction, the expansion of the Lie exponential supplies a second-order model compatible with the Cartan-Schouten connection that still guarantees the superlinear convergence. This result matches the theory in~\citep{mahony2002geometry, absil2008optimization} and verifies that our analytical gradients in~\Cref{table:diff-rotation} are derived correctly.

\subsection{Numerical Benchmark}

\begin{table*}[t]
\centering
\small
\caption{Comparison of trajectory optimization methods. We compare the direct collocation that uses \method, IPOPT, and SNOPT. We also consider the methods leveraging the temporal structures of the optimal control problem, such as single- or multiple-shooting methods with different optimization backends. We note that only \method leverages the manifold structures while the rest have to convert the problem to Euclidean space. } 
\label{tab:solver_compare}
\begin{tabular}{l c c c}
\hline
Method & Formulation & Optimization Backend & Support of Manifold Variables\\
\hline
\textbf{\method} & {} & IPM & Correct-by-Construction\\
IPOPT~\citep{wachter2006implementation}  & Direct collocation & IPM & No \\
SNOPT~\citep{gill2005snopt}  & {} & SQP & No\\ \hline
ACADO~\citep{houska2011acado}  & Multiple shooting & SQP & No\\
CROCODDYL~\citep{mastalli2020crocoddyl}  & Multiple-shooting & DDP & No\\
ALTRO~\citep{howell2019altro}  & Single-shooting & iLQR & No\\
\hline
\end{tabular}
\end{table*}

\begin{table*}[t]
\centering 
\caption{Results with the \textbf{variational integrator}. Entries in the last four columns are reported as mean$\pm$std(success mean$\pm$std). We also reported the statistics for the successful case, as many failure cases are due to running out of computation budget, i.e., the maximal number of iterations. We use a laptop with an Intel i7-11850H CPU to test the numerical benchmarks in a single core. Among those whose convergence rate is bigger than 50\%: {\color{green!60} Best} {\color{green!20} Second Best}. }
\label{tab:uniform-vi-results-std}
\footnotesize
\setlength{\tabcolsep}{5pt}
\begin{tabular}{ccccccc}
\hline
Algorithm & Noise & Converged & Avg. \#Iter & Avg. \#Iter (success) & Wall Clock (s) & Wall Clock (success) (s) \\
\hline
\multirow{4}{*}{\textbf{\method}}
& none & \cellcolor{green!60}999 & \makebox[3.2em][r]{$ 22.1 $} $\pm$ \makebox[3.2em][l]{$ 7.9 $} & \makebox[3.2em][r]{$ 22.2 $} $\pm$ \makebox[3.2em][l]{$ 7.9 $} & $ 1.56e-02 \pm 8.53e-03 $ & $ 1.56e-02 \pm 8.51e-03 $ \\
& $0.1$ & \cellcolor{green!60}999 & \cellcolor{green!20}\makebox[3.2em][r]{$ 10.2 $} $\pm$ \makebox[3.2em][l]{$ 1.9 $} & \cellcolor{green!20}\makebox[3.2em][r]{$ 10.2 $} $\pm$ \makebox[3.2em][l]{$ 1.9 $} & \cellcolor{green!60}$ 8.30e-03 \pm 1.82e-03 $ & \cellcolor{green!60}$ 8.29e-03 \pm 1.82e-03 $ \\
& $0.2$ & \cellcolor{green!60}998 & \cellcolor{green!20}\makebox[3.2em][r]{$ 12.5 $} $\pm$ \makebox[3.2em][l]{$ 31.4 $} & \cellcolor{green!20}\makebox[3.2em][r]{$ 11.5 $} $\pm$ \makebox[3.2em][l]{$ 2.4 $} & \cellcolor{green!60}$ 9.38e-03 \pm 1.97e-02 $ & \cellcolor{green!60}$ 8.76e-03 \pm 2.02e-03 $ \\
& $0.3$ & 901 & \cellcolor{green!20}\makebox[3.2em][r]{$ 13.5 $} $\pm$ \makebox[3.2em][l]{$ 4.8 $} & \cellcolor{green!20}\makebox[3.2em][r]{$ 14.6 $} $\pm$ \makebox[3.2em][l]{$ 3.5 $} & \cellcolor{green!60}$ 1.34e-02 \pm 5.24e-03 $ & \cellcolor{green!60}$ 1.20e-02 \pm 3.39e-03 $ \\
\hline
\multirow{4}{*}{IPOPT}
& none & 862 & \makebox[3.2em][r]{$ 316.0 $} $\pm$ \makebox[3.2em][l]{$ 306.0 $} & \makebox[3.2em][r]{$ 220.0 $} $\pm$ \makebox[3.2em][l]{$ 178.0 $} & $ 9.45e-01 \pm 1.05e+00 $ & $ 6.54e-01 \pm 6.64e-01 $ \\
& $0.1$ & 860 & \makebox[3.2em][r]{$ 312.0 $} $\pm$ \makebox[3.2em][l]{$ 321.0 $} & \makebox[3.2em][r]{$ 206.0 $} $\pm$ \makebox[3.2em][l]{$ 188.0 $} & $ 7.51e-01 \pm 8.87e-01 $ & $ 5.08e-01 \pm 5.99e-01 $ \\
& $0.2$ & 825 & \makebox[3.2em][r]{$ 411.0 $} $\pm$ \makebox[3.2em][l]{$ 310.0 $} & \makebox[3.2em][r]{$ 300.0 $} $\pm$ \makebox[3.2em][l]{$ 188.0 $} & $ 9.31e-01 \pm 8.73e-01 $ & $ 6.79e-01 \pm 6.06e-01 $ \\
& $0.3$ & 654 & \makebox[3.2em][r]{$ 600.0 $} $\pm$ \makebox[3.2em][l]{$ 344.0 $} & \makebox[3.2em][r]{$ 402.0 $} $\pm$ \makebox[3.2em][l]{$ 240.0 $} & $ 1.16e+00 \pm 8.62e-01 $ & $ 7.89e-01 \pm 6.07e-01 $ \\
\hline
\multirow{4}{*}{SNOPT}
& none & 920 & \makebox[3.2em][r]{$ 227.0 $} $\pm$ \makebox[3.2em][l]{$ 660.0 $} & \makebox[3.2em][r]{$ 202.0 $} $\pm$ \makebox[3.2em][l]{$ 670.0 $} & $ 3.56e-01 \pm 8.39e-01 $ & $ 3.16e-01 \pm 8.43e-01 $ \\
& $0.1$ & 952 & \makebox[3.2em][r]{$ 159.0 $} $\pm$ \makebox[3.2em][l]{$ 404.0 $} & \makebox[3.2em][r]{$ 142.0 $} $\pm$ \makebox[3.2em][l]{$ 368.0 $} & $ 2.80e-01 \pm 5.77e-01 $ & $ 2.56e-01 \pm 5.43e-01 $ \\
& $0.2$ & 935 & \makebox[3.2em][r]{$ 184.0 $} $\pm$ \makebox[3.2em][l]{$ 496.0 $} & \makebox[3.2em][r]{$ 158.0 $} $\pm$ \makebox[3.2em][l]{$ 434.0 $} & $ 3.26e-01 \pm 7.04e-01 $ & $ 2.90e-01 \pm 6.30e-01 $ \\
& $0.3$ & \cellcolor{green!20}919 & \makebox[3.2em][r]{$ 188.0 $} $\pm$ \makebox[3.2em][l]{$ 392.0 $} & \makebox[3.2em][r]{$ 157.0 $} $\pm$ \makebox[3.2em][l]{$ 310.0 $} & $ 3.37e-01 \pm 5.51e-01 $ & $ 2.91e-01 \pm 4.77e-01 $ \\
\hline
\multirow{4}{*}{CROCODDYL}
& none & 993 & \cellcolor{green!60}\makebox[3.2em][r]{$ 10.6 $} $\pm$ \makebox[3.2em][l]{$ 6.0 $} & \cellcolor{green!60}\makebox[3.2em][r]{$ 10.6 $} $\pm$ \makebox[3.2em][l]{$ 6.0 $} & \cellcolor{green!60}$ 1.62e-03 \pm 1.09e-03 $ & \cellcolor{green!60}$ 1.62e-03 \pm 1.09e-03 $ \\
& $0.1$ & 111 & \makebox[3.2em][r]{$ 3.7 $} $\pm$ \makebox[3.2em][l]{$ 6.4 $} & \makebox[3.2em][r]{$ 8.4 $} $\pm$ \makebox[3.2em][l]{$ 5.3 $} & $ 8.76e-04 \pm 1.28e-03 $ & $ 1.53e-03 \pm 1.04e-03 $ \\
& $0.2$ & 8 & \makebox[3.2em][r]{$ 0.7 $} $\pm$ \makebox[3.2em][l]{$ 3.3 $} & \makebox[3.2em][r]{$ 5.5 $} $\pm$ \makebox[3.2em][l]{$ 3.9 $} & $ 3.17e-04 \pm 7.20e-04 $ & $ 1.74e-03 \pm 9.34e-04 $ \\
& $0.3$ & 5 & \makebox[3.2em][r]{$ 0.1 $} $\pm$ \makebox[3.2em][l]{$ 1.2 $} & \makebox[3.2em][r]{$ 8.0 $} $\pm$ \makebox[3.2em][l]{$ 3.7 $} & $ 2.31e-04 \pm 3.34e-04 $ & $ 2.26e-03 \pm 1.77e-03 $ \\
\hline
\multirow{4}{*}{ACADO}
& none & \cellcolor{green!20}996 & \cellcolor{green!20}\makebox[3.2em][r]{$ 16.7 $} $\pm$ \makebox[3.2em][l]{$ 31.7 $} & \cellcolor{green!20}\makebox[3.2em][r]{$ 14.8 $} $\pm$ \makebox[3.2em][l]{$ 8.1 $} & $ 2.06e-01 \pm 6.69e-01 $ & $ 1.66e-01 \pm 2.30e-01 $ \\
& $0.1$ & \cellcolor{green!20}997 & \cellcolor{green!60}\makebox[3.2em][r]{$ 8.9 $} $\pm$ \makebox[3.2em][l]{$ 29.2 $} & \cellcolor{green!60}\makebox[3.2em][r]{$ 7.4 $} $\pm$ \makebox[3.2em][l]{$ 11.4 $} & \cellcolor{green!20}$ 1.03e-01 \pm 5.17e-01 $ & \cellcolor{green!20}$ 7.78e-02 \pm 2.22e-01 $ \\
& $0.2$ & \cellcolor{green!20}996 & \cellcolor{green!60}\makebox[3.2em][r]{$ 10.3 $} $\pm$ \makebox[3.2em][l]{$ 32.4 $} & \cellcolor{green!60}\makebox[3.2em][r]{$ 8.4 $} $\pm$ \makebox[3.2em][l]{$ 9.5 $} & \cellcolor{green!20}$ 1.18e-01 \pm 5.73e-01 $ & \cellcolor{green!20}$ 8.35e-02 \pm 1.79e-01 $ \\
& $0.3$ & \cellcolor{green!60}996 & \cellcolor{green!60}\makebox[3.2em][r]{$ 11.6 $} $\pm$ \makebox[3.2em][l]{$ 32.2 $} & \cellcolor{green!60}\makebox[3.2em][r]{$ 9.6 $} $\pm$ \makebox[3.2em][l]{$ 8.8 $} & \cellcolor{green!20}$ 1.37e-01 \pm 6.28e-01 $ & \cellcolor{green!20}$ 9.95e-02 \pm 1.99e-01 $ \\
\hline
\multirow{4}{*}{ALTRO}
& none & 621 & \makebox[3.2em][r]{$ 26.7 $} $\pm$ \makebox[3.2em][l]{$ 28.1 $} & \makebox[3.2em][r]{$ 21.8 $} $\pm$ \makebox[3.2em][l]{$ 28.2 $} & \cellcolor{green!20}$ 1.07e-02 \pm 1.21e-02 $ & \cellcolor{green!20}$ 8.80e-03 \pm 1.24e-02 $ \\
& $0.1$ & 7 & \makebox[3.2em][r]{$ 1.9 $} $\pm$ \makebox[3.2em][l]{$ 4.9 $} & \makebox[3.2em][r]{$ 51.4 $} $\pm$ \makebox[3.2em][l]{$ 20.8 $} & $ 8.84e-04 \pm 1.84e-03 $ & $ 1.72e-02 \pm 1.20e-02 $ \\
& $0.2$ & 0 & \makebox[3.2em][r]{$ 1.1 $} $\pm$ \makebox[3.2em][l]{$ 0.8 $} & $ \mathrm{---} $ & $ 6.94e-04 \pm 2.78e-04 $ & $ \mathrm{---} $ \\
& $0.3$ & 0 & \makebox[3.2em][r]{$ 1.1 $} $\pm$ \makebox[3.2em][l]{$ 0.7 $} & $ \mathrm{---} $ & $ 7.53e-04 \pm 4.39e-04 $ & $ \mathrm{---} $ \\
\hline
\end{tabular}
\end{table*}

\begin{table*}[t]
\centering 
\caption{Results with the \textbf{Euler integrator} for all the \textbf{baselines}. When using Euler integrators, we note that the quaternion variables will no longer stay on the $\|q\|=1$ manifold for the direct method. Entries in the last four columns are reported as mean$\pm$std(success mean$\pm$std). We also reported the statistics for the successful case, as many failure cases are due to running out of computation budget, i.e., the maximal number of iterations. We use a laptop with an Intel i7-11850H CPU to test the numerical benchmarks in a single core. Among those whose convergence rate is bigger than 50\%: {\color{green!60} Best} {\color{green!20} Second Best}.}
\label{tab:uniform-euler-results-std}
\footnotesize
\setlength{\tabcolsep}{5pt}
\begin{tabular}{ccccccc}
\hline
Algorithm & Noise & Converged & Avg. \#Iter & Avg. \#Iter (success) & Wall Clock (s) & Wall Clock (success) (s) \\
\hline
\multirow{4}{*}{IPOPT}
& none & \cellcolor{green!20}990 & \makebox[3.2em][r]{$ 130.0 $} $\pm$ \makebox[3.2em][l]{$ 99.5 $} & \makebox[3.2em][r]{$ 126.0 $} $\pm$ \makebox[3.2em][l]{$ 78.4 $} & $ 2.73e-01 \pm 3.28e-01 $ & $ 2.66e-01 \pm 3.14e-01 $ \\
& $0.1$ & \cellcolor{green!20}990 & \makebox[3.2em][r]{$ 110.0 $} $\pm$ \makebox[3.2em][l]{$ 122.0 $} & \makebox[3.2em][r]{$ 103.0 $} $\pm$ \makebox[3.2em][l]{$ 92.1 $} & \cellcolor{green!20}$ 2.06e-01 \pm 3.32e-01 $ & $ 1.95e-01 \pm 3.11e-01 $ \\
& $0.2$ & \cellcolor{green!20}997 & \makebox[3.2em][r]{$ 128.0 $} $\pm$ \makebox[3.2em][l]{$ 94.1 $} & \makebox[3.2em][r]{$ 127.0 $} $\pm$ \makebox[3.2em][l]{$ 90.1 $} & $ 2.31e-01 \pm 3.17e-01 $ & $ 2.30e-01 \pm 3.14e-01 $ \\
& $0.3$ & \cellcolor{green!20}994 & \makebox[3.2em][r]{$ 162.0 $} $\pm$ \makebox[3.2em][l]{$ 99.2 $} & \makebox[3.2em][r]{$ 161.0 $} $\pm$ \makebox[3.2em][l]{$ 92.0 $} & $ 2.85e-01 \pm 3.39e-01 $ & $ 2.83e-01 \pm 3.36e-01 $ \\
\hline
\multirow{4}{*}{SNOPT}
& none & 850 & \makebox[3.2em][r]{$ 163.0 $} $\pm$ \makebox[3.2em][l]{$ 217.0 $} & \makebox[3.2em][r]{$ 164.0 $} $\pm$ \makebox[3.2em][l]{$ 186.0 $} & $ 2.96e-01 \pm 4.20e-01 $ & $ 2.56e-01 \pm 3.30e-01 $ \\
& $0.1$ & 937 & \makebox[3.2em][r]{$ 157.0 $} $\pm$ \makebox[3.2em][l]{$ 218.0 $} & \makebox[3.2em][r]{$ 149.0 $} $\pm$ \makebox[3.2em][l]{$ 178.0 $} & $ 2.61e-01 \pm 3.53e-01 $ & $ 2.40e-01 \pm 2.88e-01 $ \\
& $0.2$ & 950 & \makebox[3.2em][r]{$ 172.0 $} $\pm$ \makebox[3.2em][l]{$ 239.0 $} & \makebox[3.2em][r]{$ 169.0 $} $\pm$ \makebox[3.2em][l]{$ 236.0 $} & $ 2.74e-01 \pm 3.64e-01 $ & $ 2.64e-01 \pm 3.66e-01 $ \\
& $0.3$ & 928 & \makebox[3.2em][r]{$ 183.0 $} $\pm$ \makebox[3.2em][l]{$ 318.0 $} & \makebox[3.2em][r]{$ 174.0 $} $\pm$ \makebox[3.2em][l]{$ 316.0 $} & $ 2.95e-01 \pm 4.09e-01 $ & $ 2.73e-01 \pm 3.96e-01 $ \\
\hline
\multirow{4}{*}{CROCODDYL}
& none & \cellcolor{green!60}1000 & \cellcolor{green!60}\makebox[3.2em][r]{$ 10.1 $} $\pm$ \makebox[3.2em][l]{$ 5.4 $} & \cellcolor{green!60}\makebox[3.2em][r]{$ 10.1 $} $\pm$ \makebox[3.2em][l]{$ 5.4 $} & \cellcolor{green!60}$ 1.19e-03 \pm 7.38e-04 $ & \cellcolor{green!60}$ 1.19e-03 \pm 7.38e-04 $ \\
& $0.1$ & \cellcolor{green!60}1000 & \cellcolor{green!60}\makebox[3.2em][r]{$ 11.6 $} $\pm$ \makebox[3.2em][l]{$ 5.2 $} & \cellcolor{green!20}\makebox[3.2em][r]{$ 11.6 $} $\pm$ \makebox[3.2em][l]{$ 5.2 $} & \cellcolor{green!60}$ 1.40e-03 \pm 9.39e-04 $ & \cellcolor{green!60}$ 1.40e-03 \pm 9.39e-04 $ \\
& $0.2$ & \cellcolor{green!60}1000 & \cellcolor{green!60}\makebox[3.2em][r]{$ 12.1 $} $\pm$ \makebox[3.2em][l]{$ 7.1 $} & \cellcolor{green!20}\makebox[3.2em][r]{$ 12.1 $} $\pm$ \makebox[3.2em][l]{$ 7.1 $} & \cellcolor{green!60}$ 1.48e-03 \pm 9.74e-04 $ & \cellcolor{green!60}$ 1.48e-03 \pm 9.74e-04 $ \\
& $0.3$ & \cellcolor{green!60}1000 & \cellcolor{green!60}\makebox[3.2em][r]{$ 11.5 $} $\pm$ \makebox[3.2em][l]{$ 7.2 $} & \cellcolor{green!20}\makebox[3.2em][r]{$ 11.5 $} $\pm$ \makebox[3.2em][l]{$ 7.2 $} & \cellcolor{green!60}$ 1.30e-03 \pm 9.16e-04 $ & \cellcolor{green!60}$ 1.30e-03 \pm 9.16e-04 $ \\
\hline
\multirow{4}{*}{ACADO}
& none & 982 & \cellcolor{green!20}\makebox[3.2em][r]{$ 23.0 $} $\pm$ \makebox[3.2em][l]{$ 65.4 $} & \cellcolor{green!20}\makebox[3.2em][r]{$ 14.3 $} $\pm$ \makebox[3.2em][l]{$ 10.7 $} & $ 3.06e-01 \pm 1.23e+00 $ & $ 1.45e-01 \pm 2.50e-01 $ \\
& $0.1$ & 977 & \cellcolor{green!20}\makebox[3.2em][r]{$ 20.6 $} $\pm$ \makebox[3.2em][l]{$ 75.4 $} & \cellcolor{green!60}\makebox[3.2em][r]{$ 9.3 $} $\pm$ \makebox[3.2em][l]{$ 16.9 $} & $ 3.01e-01 \pm 1.38e+00 $ & \cellcolor{green!20}$ 9.49e-02 \pm 2.65e-01 $ \\
& $0.2$ & 981 & \cellcolor{green!20}\makebox[3.2em][r]{$ 19.5 $} $\pm$ \makebox[3.2em][l]{$ 68.6 $} & \cellcolor{green!60}\makebox[3.2em][r]{$ 10.2 $} $\pm$ \makebox[3.2em][l]{$ 15.5 $} & $ 2.69e-01 \pm 1.25e+00 $ & $ 1.00e-01 \pm 2.35e-01 $ \\
& $0.3$ & 982 & \cellcolor{green!20}\makebox[3.2em][r]{$ 19.6 $} $\pm$ \makebox[3.2em][l]{$ 66.5 $} & \cellcolor{green!60}\makebox[3.2em][r]{$ 10.8 $} $\pm$ \makebox[3.2em][l]{$ 13.8 $} & $ 2.72e-01 \pm 1.23e+00 $ & $ 1.12e-01 \pm 2.84e-01 $ \\
\hline
\multirow{4}{*}{ALTRO}
& none & 704 & \makebox[3.2em][r]{$ 32.0 $} $\pm$ \makebox[3.2em][l]{$ 36.5 $} & \makebox[3.2em][r]{$ 29.8 $} $\pm$ \makebox[3.2em][l]{$ 38.0 $} & \cellcolor{green!20}$ 1.00e-02 \pm 1.18e-02 $ & \cellcolor{green!20}$ 9.30e-03 \pm 1.24e-02 $ \\
& $0.1$ & 473 & \makebox[3.2em][r]{$ 79.2 $} $\pm$ \makebox[3.2em][l]{$ 53.3 $} & \makebox[3.2em][r]{$ 79.1 $} $\pm$ \makebox[3.2em][l]{$ 55.9 $} & $ 2.28e-02 \pm 1.78e-02 $ & $ 2.21e-02 \pm 1.83e-02 $ \\
& $0.2$ & 528 & \makebox[3.2em][r]{$ 85.6 $} $\pm$ \makebox[3.2em][l]{$ 58.1 $} & \makebox[3.2em][r]{$ 84.1 $} $\pm$ \makebox[3.2em][l]{$ 61.5 $} & \cellcolor{green!20}$ 2.87e-02 \pm 2.33e-02 $ & \cellcolor{green!20}$ 2.78e-02 \pm 2.49e-02 $ \\
& $0.3$ & 586 & \makebox[3.2em][r]{$ 95.9 $} $\pm$ \makebox[3.2em][l]{$ 63.6 $} & \makebox[3.2em][r]{$ 99.4 $} $\pm$ \makebox[3.2em][l]{$ 67.8 $} & \cellcolor{green!20}$ 2.70e-02 \pm 2.11e-02 $ & \cellcolor{green!20}$ 2.81e-02 \pm 2.27e-02 $ \\
\hline
\end{tabular}
\end{table*}

 \begin{figure*}
    \centering
    \includegraphics[width=1\linewidth]{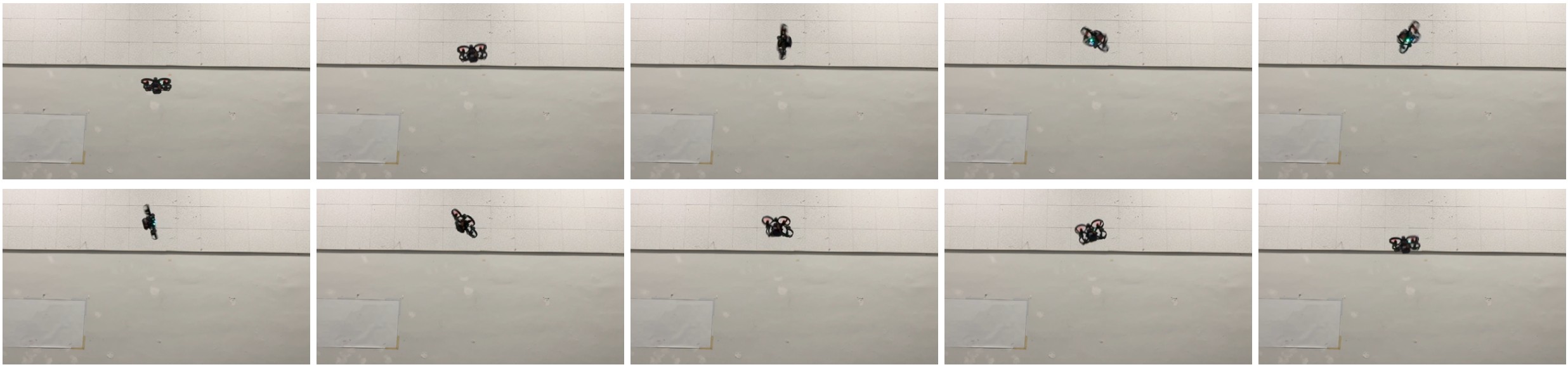}
    \caption{Snapshots of the in-place flip experiments. The quadrotor was initialized at the hovering pose and then moved upward, followed by a back flip motion, and then recovered to the hovering pose. This motion traverses the singular pose where the pitch angle equals $\pm 90^\circ$.}
    \label{fig:flip}
\end{figure*}

\begin{figure*}
    \centering
    \includegraphics[width=1\linewidth]{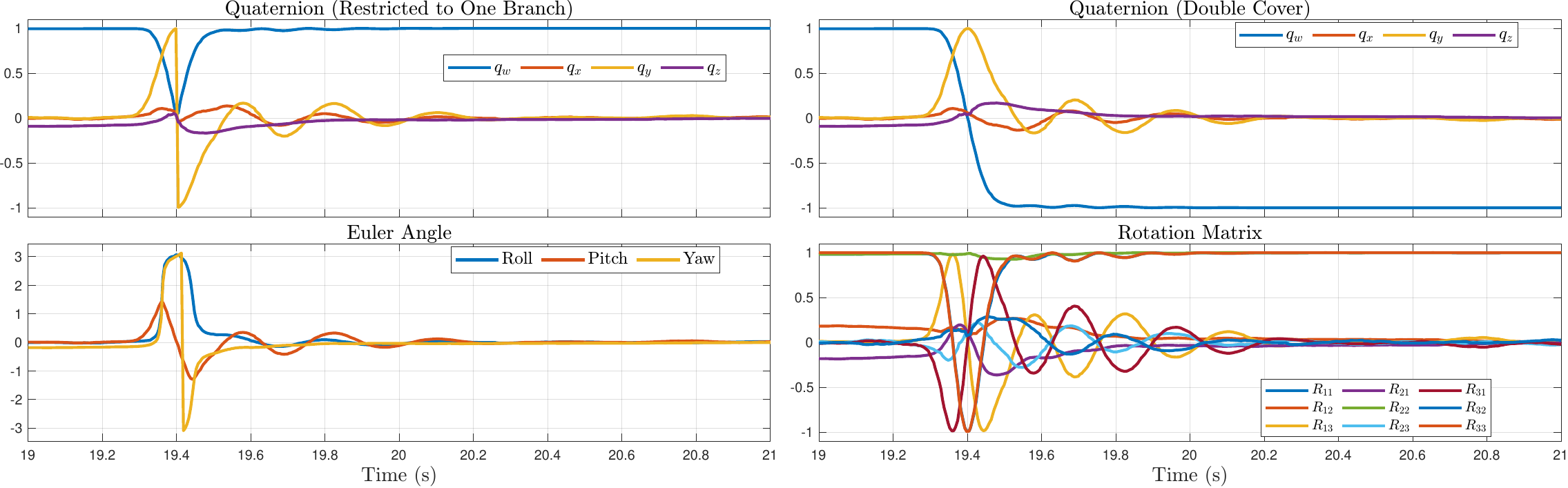}
    \caption{
    The orientation of the in-place flip motion. Though a quaternion is a singularity-free representation, the representation can still be discontinuous if the branch of the quaternion is not chosen properly. Due to the double cover of the quaternion, the terminal pose moves to another branch after a full circle of rotation. The Euler angle exhibits a singularity when the pitch angle across the $90\deg$. The rotation matrix does not have the two problems as its representation for $\mathrm{SO}(3)$ is unique and free of singularity.   
    }
    \label{fig:flip}
\end{figure*}





In this section, we further compare the proposed method with the existing solvers on a quadrotor motion planning problem. We consider a quadrotor modeled as a full rigid body with input limits as illustrated in~\Cref{fig:drone}. The input of the quadrotor is the motor force $f \in \mathbb{R}^4$, which is mapped to the torque $\tau \in \mathbb{R}^3$ in the body frame and the linear force $u_z\in\mathbb{R}$ in the body $z$ axis via a linear map:
\begin{equation}
    \begin{bmatrix}
        \tau \\
        u_z
    \end{bmatrix} = \begin{bmatrix}
        -B_1& -B_1&  B_1&  B_1 \\
        -B_1&  B_1&  B_1&  -B_1 \\
        -B_2&  B_2&  -B_2&  B_2 \\
          B_3&  B_3&  B_3&  B_3 \\
    \end{bmatrix}
\begin{bmatrix}
        f_1\\
        f_2 \\
        f_3\\
        f_4
    \end{bmatrix}.
\end{equation}
The motor force has the box constraints $f \in [0, f_{\max}]^4$. 
Thus we have the continuous-time dynamics on $\mathrm{SO}(3)\times \mathbb{R}^3$ as the following equations:
\begin{equation}
\begin{aligned}
    \dot{R}&=R\omega^\wedge, \quad \dot{p} = v, \\
    I\dot{\omega}&=(I\omega) \times \omega + \tau,\ m\dot{v}= mg + Re_zu_z.
\end{aligned}
\end{equation}
By the LGVI, we have the discrete dynamics that is compatible with the input models:
\begin{equation}
\label{eq:drone_dynamics}
\begin{aligned}
    &F_{k+1}I^b - I^bF_{k+1}^{\top}=I^bF_{k} - F_k^{\top}I^b + \tau^{\wedge} \Delta t^2 , \\ 
    &mv_{k+1} = mv_{k} + mg\Delta t + R_{k+1}e_zu_z\Delta t, \\ 
\end{aligned}
\end{equation}

As summarized in~\Cref{table:rotation-coordinate-comparison}, we consider a quaternion-based formulation as illustrated in the last two columns for the baseline. We note that the ambient space representation of the rotation is $R\in \mathbb{R}^{3\times3}$, which is $9-$dimensional. As it dramatically increases dimensionality, we do not consider such an ambient space representation in our baseline. 



We then compare the performance of the \method with the existing optimization frameworks. For direct methods, we consider the general-purpose solvers IPOPT~\citep{wachter2006implementation} and SNOPT~\citep{gill2005snopt}. We note that IPOPT is an implementation of the line search filter IPM that is mostly related to our work. The SNOPT is based on active-set Sequential Quadratic Programming (SQP). We also consider a few frameworks leveraging the structure of optimal control problems, such as ACADO~\citep{houska2011acado}, ALTRO~\citep{howell2019altro}, and CROCODDYL~\citep{mastalli2020crocoddyl}. The problem formulation and optimization backend of these methods are summarized in~\Cref{tab:solver_compare}. The convergence criterion of \method and the baselines are listed in Appendix.~\ref{appx:baseline}.





We consider the task of stabilizing the quadrotor from randomly sampled initial poses to the desired positions. The location of the quadrotor is uniformly sampled from a box. The orientation is also randomly sampled such that $\log{(R_0)}$ is uniform in the interval $[0, \pi]$. We consider solving the~\eqref{eq:TO} with 21 collocation points. We consider~\eqref{eq:R=q} to ensure that the cost for the quaternion-based and matrix-based methods have equivalent cost functions. 

We uniformly interpolate between the initial and goal state to obtain kinematically feasible initializations for each method. The dual initialization strategy for the IPM-based method is to set the dual variables for equality constraints to $0$, and those for inequality constraints and slack variables to a small positive number. To test the robustness of all the proposed methods, we also add noise to the solutions found in the first test. We perturb the primal solutions in $\mathrm{SO}(3)\times\mathbb{R}^3$, such that
\begin{equation}
\begin{aligned}
        R_\epsilon=R\exp(\epsilon_R),\ \epsilon_p \sim N(0, \Sigma), \\
        p_\epsilon=p + \epsilon_p,\ \epsilon_R \sim N(0, \Sigma),
\end{aligned}
\end{equation}
for both the pose and the velocities. For the baselines, we convert the perturbed state in $\mathrm{SO}(3)$ to quaternions. For the IPM-based method, we also perturb the dual variables by additive Gaussian noise for those of equality constraints, and by multiplicative noise for inequalities and slack variables, i.e., $z_\epsilon = z \exp(\epsilon_z), s_\epsilon=s\exp(\epsilon_s)$ with $\epsilon_z,\epsilon_s \sim N(0, \Sigma)$ the Gaussian noise.

\begin{figure}
    \centering
    \includegraphics[width=1\linewidth]{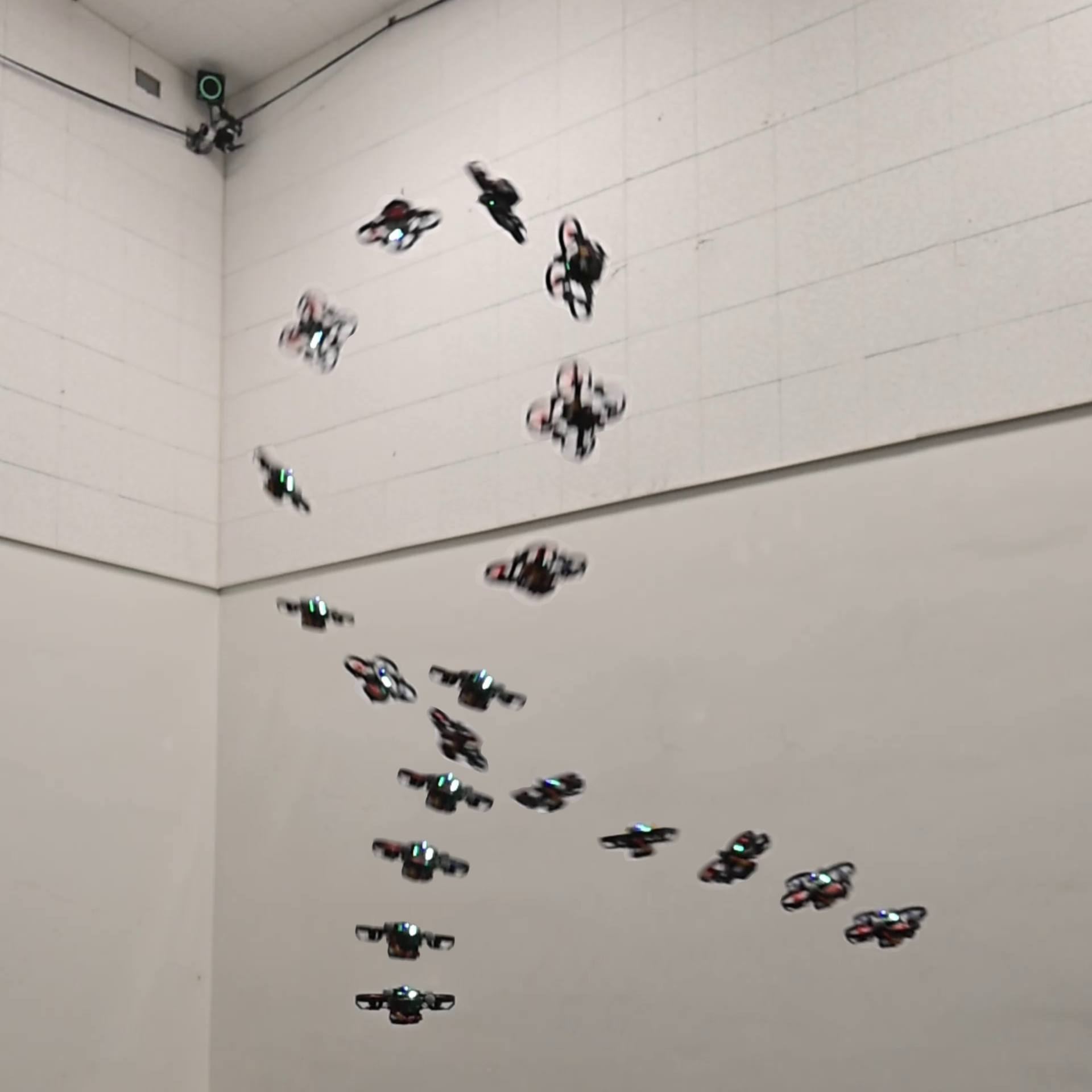}
    \caption{Snapshots of the power loop experiments. The quadrotor is commanded to rotate about the $y-$axis for a full circle while tracking motions in the $x-z$ plane. The quadrotor starts to flip from the bottom-center to the bottom-right of the picture. }
    \label{fig:powerloop}
\end{figure}

\begin{figure*}
    \centering
    \includegraphics[width=1\linewidth]{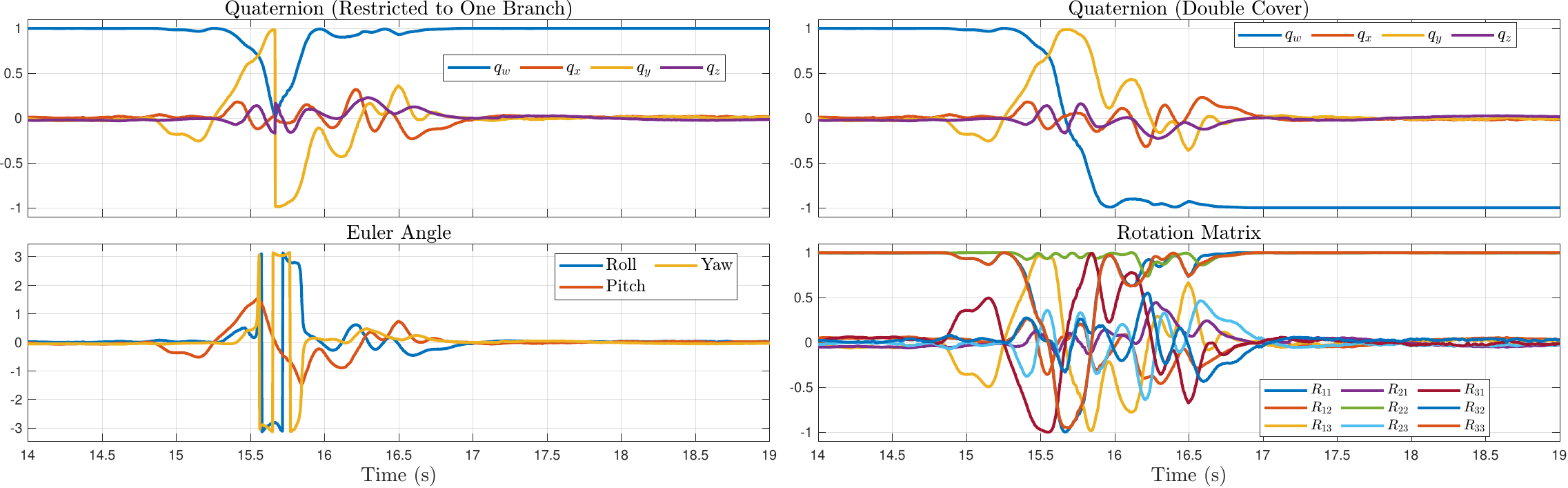}
    \caption{
    The orientation of the power-loop motion. Due to the translational motion, the rotation has more oscillation compared to the in-place motion. The trajectories of orientation represented by the rotation matrix remain continuous along the trajectory.  
    }
    \label{fig:loop}
\end{figure*}

\begin{figure*}
    \centering
    \includegraphics[width=1\linewidth]{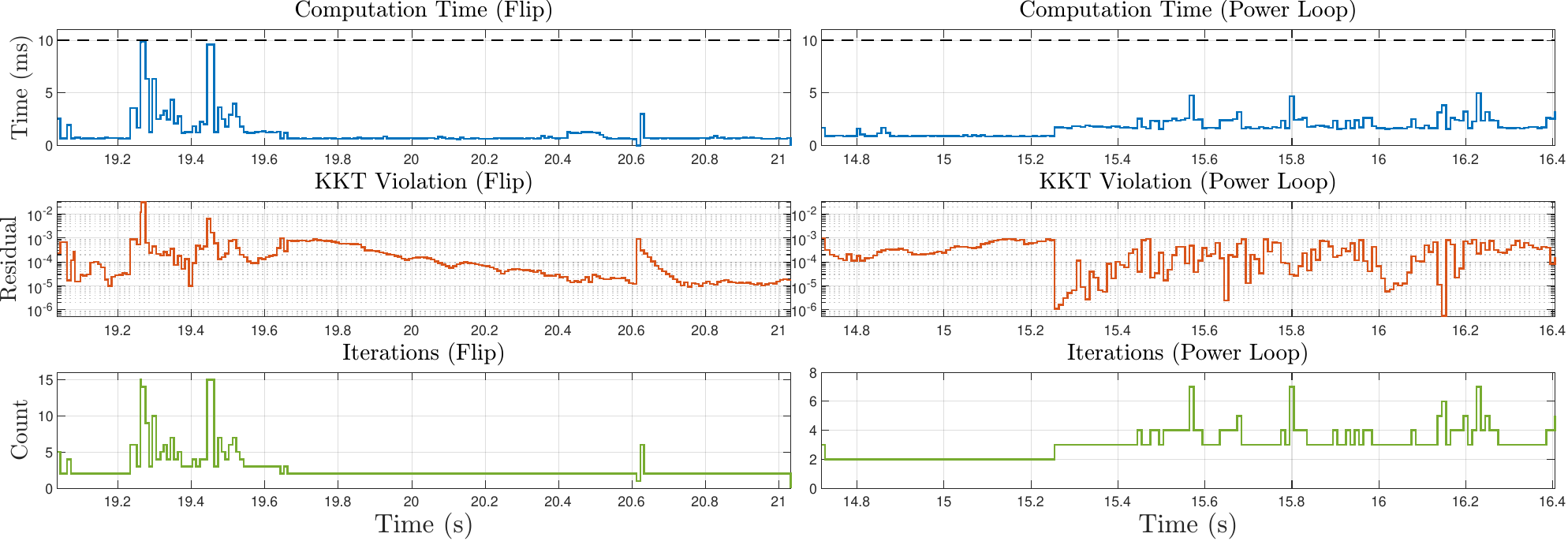}
    \caption{The computational time and KKT residual in the runtime. The rate of MPC control loop is set to be 100Hz. We find that in the power-loop case, \method only takes half of 0.01s to finish the computation. In the flip case, the solver only reaches the time budget when the desired roll angle changes abruptly.}
    \label{fig:kkt_time}
\end{figure*}

As summarized in~\Cref{tab:uniform-vi-results-std} for dynamics based on variational integrators, our method has the highest convergence rate among all the methods with noise up to $0.2$. Among the direct methods, \method can save up to $90\%$ of iterations when compared to SNOPT or IPOPT. The wall clock consumption of IPOPT or SNOPT is 10 to 20 times that of \method. For the shooting methods, we note that ACADO is the most robust one with a higher convergence rate than that of ALTRO and CROCODDYL. However, we note that though ACADO takes fewer iterations, the wall clock time is higher as each iteration includes repeated interval integration, sensitivity propagation, condensing, and QP solves. For the CROCODDYL, we note that the performance degrades a lot when noise is added. The single shooting method ALTRO is overall less robust than the multi-shooting method, i.e., CROCODDYL and ACADO. 

We later applied the baselines using the explicit Euler integration as summarized in~\Cref{tab:uniform-euler-results-std}. We find that all the methods exhibit a higher convergence rate compared to the variational integrators. The CROCODDYL exhibits extremely good performance in all noise levels. However, we note that the explicit Euler method sacrifices the manifold constraints where the quaternions can leave the $\|q\|=1$ manifold, which is not correct for rotation. 

In conclusion, \method is the only method that can naturally preserve the Lie group structure while maintaining high computational efficiency.  
Among the direct method solvers, \method takes a magnitude fewer iterations and shorter wall clock. 
When compared with the shooting-based method, \method still outperforms the three baselines in terms of wall clock and robustness without sacrificing the manifold structure.

\subsection{Applications on Quadrotor Aerobatics}

Now we proceed to deploy \method in the real-world experiments. We consider applying \method to solve \eqref{eq:TO} in a receding horizon manner for quadrotor motion planning. The \eqref{eq:TO} in the MPC implementation considers 21 collocation points. To balance the planning horizon and the number of collocation points, we consider a finer time step at the start of the horizon and a longer time step at the tail of the horizon.


\textbf{Hardware setup: }Figure~\ref{fig:drone} shows the quadrotor platform used in the real-world experiments. 
The platform is equipped with four 1106 7500kV BLDC motors with 2040-3 propellers.
The onboard Pixracer R15 flight control unit~(FCU) provides inertial measurements, including linear acceleration and body angular velocity. The vehicle position and attitude are measured by a motion-capture system and fused with onboard sensing through an extended Kalman filter to obtain the state estimate. During flight,~\eqref{eq:TO} is solved in real time by \method on an off-board personal computer. 

\textbf{Control architecture: }For real-time execution, the control rate of MPC is set to $100$~Hz. The tolerance of \method is set to $E_0 \le \epsilon_{\mathrm{tol}} = 10^{-3}$. To ensure bounded computation time within the control loop, \method will exit the main loop of~\Cref{alg:line-search-RIPM} if the computation time exceeds $8.5$~ms. After the MPC is solved, the solved trajectory is sent to the on-board PD controller for tracking~\citep{lee2010geometric}. 



\textbf{In-place Flip:} In the first experiment, we consider the flip motion such that we command the quadrotor pitch to $180^\circ$ and then recover to the hovering height $z_h$ using the following command:
\begin{equation}
\begin{aligned}
    z_{d} &= \left\{ \begin{aligned}
        &z_{h} + 0.4,\quad &0 \le t \le 0.2 \\
                    &z_h  \quad &t > 0.2
    \end{aligned}\right. \\
    \theta_{y} &= \left\{ \begin{aligned}
        &\pi,\quad &0 \le t \le 0.2 \\
                    &0  \quad &t > 0.2
    \end{aligned}\right.
\end{aligned}
\end{equation}

The motion of the quadrotor is illustrated in~\Cref{fig:flip}. As shown in \Cref{fig:flip}, \method successfully enables the quadrotor to change the pitch angles across $\pm 90^\circ$. After the transient state, the quadrotor resumes its hovering pose with zero inclination angle at the desired height. 

\textbf{Power-loop: } In the second experiment, we consider the power loop motion such that we command the quadrotor to rotate $360^\circ$ around the $y$ axis while tracking a circle in the $x-z$ plane. We first generate a dynamically feasible reference trajectory by solving \eqref{eq:TO} given a kinematically feasible trajectory interpolated between key frames. Then the resulting dynamically feasible trajectory is given to the MPC for tracking. 

The snapshots of the experiment are shown in~\Cref{fig:powerloop}. As shown in \Cref{fig:loop}, the quadrotor successfully traverses $\pm90^\circ$ pitch angle.

\textbf{Computation time: }We further report the computation time and convergence of the \method in~\Cref{fig:kkt_time}. In both cases, \method converges far below the control interval $10$ ms. In the flip motion, only the case with a large change of reference pitch angle consumes more time. For both motions, the KKT violation in most of the cases converges below $1e^{-3}$.






\section{Discussions}
\label{sec:lim}


\textbf{Global Convergence:} 
We note that the proposed method is not a full manifold extension of the filter line-search framework used in IPOPT~\citep{wachter2006implementation}. Instead, we adopt a simpler globalization strategy based on backtracking line search and inertia correction. Though this is not a full filter-based method with infeasibility restoration as in~\citep{wachter2006implementation}, we find that the current version already outperforms IPOPT or SNOPT with an order of magnitude in convergence. In the manifold setting, \method also outperforms the shooting-based baselines that utilized the structure of the optimal control problem. 

As a consequence of the lack of feasibility restoration, we have to set a maximum number of line searches to terminate the line search. In IPOPT, the filter will determine when to terminate the line search. Extending the filter or restoration mechanisms to the manifold setting, while preserving geometric consistency, it remains an interesting direction for future work.

\textbf{Topology of $\mathrm{SO}(3)$ and $\mathrm{SU}(2)$:}
The matrix Lie group $\mathrm{SO}(3)$ and the unit-quaternion group $\mathrm{SU}(2)$ both provide singularity-free representations of three-dimensional rotations. However, their global topologies are different. The rotation group satisfies
$\mathrm{SO}(3)\cong \mathbb{RP}^3 \cong \mathrm{S}^3/\{q\sim -q\}$,
where the three-dimensional real projective space $\mathbb{RP}^3$ can be viewed as the space of lines through the origin in $\mathbb{R}^4$. In contrast, $\mathrm{SU}(2)\cong \mathrm{S}^3$ keeps the two antipodal points $q$ and $-q$ distinct, although they correspond to the same physical rotation. Thus, $\mathrm{SU}(2)$ is a double cover of $\mathrm{SO}(3)$ rather than a one-to-one rotation space. As shown in~\Cref{fig:teaser}, an optimization problem with a unique solution on $\mathrm{SO}(3)$ may lift to two antipodal minimizers on $\mathrm{SU}(2)$, which are distinct in the covering space but equivalent after projection to $\mathrm{SO}(3)$. As shown in~\Cref{fig:flip} and~\ref{fig:loop}, the quaternion will not be continuous if one restricts it to one branch. The double cover issue results in an additional branch selection issue when designing continuous trajectories. 

\textbf{Optimization Landscape: }As illustrated in~\Cref{fig:teaser}, optimizing on $\mathrm{SO}(3)$ results in a homogeneous symmetric space that is less nonlinear than the generalized coordinates, such as the Euler angles. Though intuitively, this phenomenon may lead to a more friendly landscape for optimization, we leave the formal analysis as an interesting future research question. Such an analysis can potentially be applied for broader robotics applications ~\citep{teng2022input, teng2024convex, 10301632, teng2021toward, liu2025discrete, teng2024generalized, Teng-RSS-23, ghaffari2022progress, teng2022lie, teng2022error, teng2021legged, yu2023fully, he2024legged, teng2024gmkf, he2024legged,chang2026survey,teng2026max,liu2025ego, he2025invariant, iwasaki2025learning, liu2026mepoly, dong2025online, li2026stein}

\section{Conclusions}
\label{sec:conclusion}

In this work, we propose the \methodlong (\method) on matrix Lie groups for direct trajectory optimization problems of rigid bodies. This work derives the second-order models for rigid body dynamics on matrix Lie groups, leveraging the symmetry of the Lie group mechanics and kinematics. Then we developed \method, a back-tracking line search interior point method on matrix Lie groups. The \method is correct-by-construction to preserve the topological structure of matrix Lie groups with provable local superlinear convergence. We show that \method, when properly using the on-manifold derivatives, outperforms the baselines by an order of magnitude in terms of convergence rate and computation time in direct trajectory optimization of rigid bodies. Finally, \method is deployed for real-time Model Predictive Control in quadrotor aerobatic tasks.



\section*{Acknowledgment}
The work by S. Teng and K. Sreenath was supported by The Robotics and AI Institute. The work by W. Clark, R. Vasudevan, and M. Ghaffari were supported by AFOSR MURIFA9550-23-1-0400.

We thank Prof. Robert Mahony for the insightful discussions on the choice of affine connections on matrix Lie groups. We thank Dr. Thomas Lew for insightful discussions on the trajectory optimization in the ambient space. We thank Dr. Yuman Gao for helpful instructions on the experiments.

{
\balance
\bibliographystyle{SageH}
\bibliography{strings-full,ieee-full,references}

@STRING{IEEE_J_AC         = "{IEEE} Trans. Autom. Control"}

@STRING{IEEE_J_AC         = "{IEEE} Transactions on Automatic Control"}

@article{teng2026max,
  title={Max entropy moment kalman filter for polynomial systems with arbitrary noise},
  author={Teng, Sangli and Zhang, Harry and Jin, David and Jasour, Ashkan and Vasudevan, Ram and Ghaffari, Maani and Carlone, Luca},
  journal={Advances in Neural Information Processing Systems},
  volume={38},
  pages={118700--118722},
  year={2026}
}

@article{li2026stein,
  title={Stein Variational Ergodic Surface Coverage with SE (3) Constraints},
  author={Li, Jiayun and Jin, Yufeng and Teng, Sangli and Gong, Dejian and Chalvatzaki, Georgia},
  journal={arXiv preprint arXiv:2603.09458},
  year={2026}
}

@article{liu2026mepoly,
  title={MePoly: Max Entropy Polynomial Policy Optimization},
  author={Liu, Hang and Teng, Sangli and Ghaffari, Maani},
  journal={arXiv preprint arXiv:2602.17832},
  year={2026}
}

@article{dong2025online,
  title={Online Learning-Enhanced Lie Algebraic MPC for Robust Trajectory Tracking of Autonomous Surface Vehicles},
  author={Dong, Yinan and Xu, Ziyu and Lazouski, Tsimafei and Teng, Sangli and Ghaffari, Maani},
  journal={arXiv preprint arXiv:2511.18683},
  year={2025}
}

@article{iwasaki2025learning,
  title={Learning hybrid dynamics via convex optimizations},
  author={Iwasaki, Kaito and Teng, Sangli and Bloch, Anthony and Ghaffari, Maani},
  journal={arXiv preprint arXiv:2509.24157},
  year={2025}
}

@article{he2025invariant,
  title={Invariant Filtering for Full-State Estimation of Ground Robots in Noninertial Environments},
  author={He, Zijian and Teng, Sangli and Lin, Tzu-Yuan and Ghaffari, Maani and Gu, Yan},
  journal={IEEE/ASME Transactions on Mechatronics},
  year={2025},
  publisher={IEEE}
}

@article{chang2026survey,
  title={A Survey of Legged Robotics in Non-Inertial Environments: Past, Present, and Future},
  author={Chang, I and Huang, Xinyan and Lin, Tzu-Yuan and Teng, Sangli and Li, Wenjing and Ghaffari, Maani and Yi, Jingang and Gu, Yan and others},
  journal={arXiv preprint arXiv:2604.20990},
  year={2026}
}

@article{liu2025ego,
  title={Ego-Vision World Model for Humanoid Contact Planning},
  author={Liu, Hang and Gao, Yuman and Teng, Sangli and Chi, Yufeng and Shao, Yakun Sophia and Li, Zhongyu and Ghaffari, Maani and Sreenath, Koushil},
  journal={arXiv preprint arXiv:2510.11682},
  year={2025}
}

@article{teng2025chyll,
  title={CHyLL: Learning Continuous Neural Representations of Hybrid Systems},
  author={Teng, Sangli and Liu, Hang and Song, Jingyu and Sreenath, Koushil},
  journal={arXiv preprint arXiv:2512.10117},
  year={2025}
}

@article{duong2024port,
  title={Port-Hamiltonian neural ODE networks on Lie groups for robot dynamics learning and control},
  author={Duong, Thai and Altawaitan, Abdullah and Stanley, Jason and Atanasov, Nikolay},
  journal={IEEE Transactions on Robotics},
  volume={40},
  pages={3695--3715},
  year={2024},
  publisher={IEEE}
}

@article{duong2022adaptive,
  title={Adaptive control of SE (3) Hamiltonian dynamics with learned disturbance features},
  author={Duong, Thai and Atanasov, Nikolay},
  journal={IEEE Control Systems Letters},
  volume={6},
  pages={2773--2778},
  year={2022},
  publisher={IEEE}
}

@incollection{hall2013lie,
  title={Lie groups, Lie algebras, and representations},
  author={Hall, Brian C},
  booktitle={Quantum Theory for Mathematicians},
  pages={333--366},
  year={2013},
  publisher={Springer}
}

@book{hirsch2012differential,
  title={Differential topology},
  author={Hirsch, Morris W},
  year={2012},
  publisher={Springer Science \& Business Media}
}

@book{barfoot2024state,
  title={State estimation for robotics},
  author={Barfoot, Timothy D},
  year={2024},
  publisher={Cambridge University Press}
}

@article{sun2025agile,
  title={Agile and cooperative aerial manipulation of a cable-suspended load},
  author={Sun, Sihao and Wang, Xuerui and Sanalitro, Dario and Franchi, Antonio and Tognon, Marco and Alonso-Mora, Javier},
  journal={Science Robotics},
  volume={10},
  number={107},
  pages={eadu8015},
  year={2025},
  publisher={American Association for the Advancement of Science}
}

@article{nan2022nonlinear,
  title={Nonlinear MPC for quadrotor fault-tolerant control},
  author={Nan, Fang and Sun, Sihao and Foehn, Philipp and Scaramuzza, Davide},
  journal={IEEE Robotics and Automation Letters},
  volume={7},
  number={2},
  pages={5047--5054},
  year={2022},
  publisher={IEEE}
}

@article{bhat2000topological,
  title={A topological obstruction to continuous global stabilization of rotational motion and the unwinding phenomenon},
  author={Bhat, Sanjay P and Bernstein, Dennis S},
  journal={Systems \& control letters},
  volume={39},
  number={1},
  pages={63--70},
  year={2000},
  publisher={Elsevier}
}

@phdthesis{teng2025optimization,
  title={Optimization-based Robot Control and State Estimation on Matrix Lie Groups},
  author={Teng, Sangli},
  year={2025}
}

@book{absil2008optimization,
  title={Optimization algorithms on matrix manifolds},
  author={Absil, P-A and Mahony, Robert and Sepulchre, Rodolphe},
  year={2008},
  publisher={Princeton University Press}
}

@article{mahony2002geometry,
  title={The geometry of the Newton method on non-compact Lie groups},
  author={Mahony, Robert and Manton, Jonathan H},
  journal={Journal of Global Optimization},
  volume={23},
  number={3},
  pages={309--327},
  year={2002},
  publisher={Springer}
}

@inproceedings{mastalli2020crocoddyl,
  title={Crocoddyl: An efficient and versatile framework for multi-contact optimal control},
  author={Mastalli, Carlos and Budhiraja, Rohan and Merkt, Wolfgang and Saurel, Guilhem and Hammoud, Bilal and Naveau, Maximilien and Carpentier, Justin and Righetti, Ludovic and Vijayakumar, Sethu and Mansard, Nicolas},
  booktitle={2020 IEEE International Conference on Robotics and Automation (ICRA)},
  pages={2536--2542},
  year={2020},
  organization={IEEE}
}

@article{sun2022comparative,
  title={A comparative study of nonlinear mpc and differential-flatness-based control for quadrotor agile flight},
  author={Sun, Sihao and Romero, Angel and Foehn, Philipp and Kaufmann, Elia and Scaramuzza, Davide},
  journal={IEEE Transactions on Robotics},
  volume={38},
  number={6},
  pages={3357--3373},
  year={2022},
  publisher={IEEE}
}

@article{bonalli2019trajectory,
  title={Trajectory optimization on manifolds: A theoretically-guaranteed embedded sequential convex programming approach},
  author={Bonalli, Riccardo and Bylard, Andrew and Cauligi, Abhishek and Lew, Thomas and Pavone, Marco},
  journal={arXiv preprint arXiv:1905.07654},
  year={2019}
}

@article{amestoy2019performance,
  title={Performance and scalability of the block low-rank multifrontal factorization on multicore architectures},
  author={Amestoy, Patrick R and Buttari, Alfredo and L'excellent, Jean-Yves and Mary, Theo},
  journal={ACM Transactions on Mathematical Software (TOMS)},
  volume={45},
  number={1},
  pages={1--26},
  year={2019},
  publisher={ACM New York, NY, USA}
}

@INPROCEEDINGS{TengS-RSS-25, 
    AUTHOR    = {Sangli Teng AND Tzu-Yuan Lin AND William A. Clark AND Ram Vasudevan AND Maani Ghaffari}, 
    TITLE     = {{Riemannian Direct Trajectory Optimization of Rigid Bodies on Matrix Lie Groups}}, 
    BOOKTITLE = {Proceedings of Robotics: Science and Systems}, 
    YEAR      = {2025}, 
    ADDRESS   = {Los Angeles, CA, USA}, 
    MONTH     = {June}, 
    DOI       = {10.15607/RSS.2025.XXI.120} 
}

@article{he2024legged,
  title={Legged robot state estimation within non-inertial environments},
  author={He, Zijian and Teng, Sangli and Lin, Tzu-Yuan and Ghaffari, Maani and Gu, Yan},
  journal={arXiv preprint arXiv:2403.16252},
  year={2024}
}

@inproceedings{teng2022input,
  title={Input Influence Matrix Design for MIMO Discrete-Time Ultra-Local Model},
  author={Teng, Sangli and Sanyal, Amit K and Vasudevan, Ram and Bloch, Anthony and Ghaffari, Maani},
  booktitle={2022 American Control Conference (ACC)},
  pages={2730--2735},
  year={2022},
  organization={IEEE}
}

@article{teng2021toward,
  title={Toward safety-aware informative motion planning for legged robots},
  author={Teng, Sangli and Gong, Yukai and Grizzle, Jessy W and Ghaffari, Maani},
  journal={arXiv preprint arXiv:2103.14252},
  year={2021}
}

@article{liu2025discrete,
  title={Discrete-Time Hybrid Automata Learning: Legged Locomotion Meets Skateboarding},
  author={Liu, Hang and Teng, Sangli and Liu, Ben and Zhang, Wei and Ghaffari, Maani},
  journal={arXiv preprint arXiv:2503.01842},
  year={2025}
}

@article{teng2024generalized,
  title={A Generalized Metriplectic System via Free Energy and System\~{} Identification via Bilevel Convex Optimization},
  author={Teng, Sangli and Iwasaki, Kaito and Clark, William and Yu, Xihang and Bloch, Anthony and Vasudevan, Ram and Ghaffari, Maani},
  journal={arXiv preprint arXiv:2410.06233},
  year={2024}
}

@Inbook{Lee2003,
author="Lee, John M.",
title="Smooth Manifolds",
bookTitle="Introduction to Smooth Manifolds",
year="2003",
publisher="Springer New York",
address="New York, NY",
pages="1--29",
isbn="978-0-387-21752-9",
doi="10.1007/978-0-387-21752-9_1",
}

@article{han2025building,
  title={Building Rome with Convex Optimization},
  author={Han, Haoyu and Yang, Heng},
  journal={arXiv preprint arXiv:2502.04640},
  year={2025}
}

@article{wang2025unlocking,
  title={Unlocking aerobatic potential of quadcopters: Autonomous freestyle flight generation and execution},
  author={Wang, Mingyang and Wang, Qianhao and Wang, Ze and Gao, Yuman and Wang, Jingping and Cui, Can and Li, Yuan and Ding, Ziming and Wang, Kaiwei and Xu, Chao and others},
  journal={Science Robotics},
  volume={10},
  number={101},
  pages={eadp9905},
  year={2025},
  publisher={American Association for the Advancement of Science}
}

@article{hargraves1987direct,
  title={Direct trajectory optimization using nonlinear programming and collocation},
  author={Hargraves, Charles R and Paris, Stephen W},
  journal={Journal of guidance, control, and dynamics},
  volume={10},
  number={4},
  pages={338--342},
  year={1987}
}

@article{yang2014optimality,
  title={Optimality conditions for the nonlinear programming problems on {Riemannian} manifolds},
  author={Yang, Wei Hong and Zhang, Lei-Hong and Song, Ruyi},
  journal={Pacific Journal of Optimization},
  volume={10},
  number={2},
  pages={415--434},
  year={2014}
}

@article{boumal2014manopt,
  title={{Manopt}, a {Matlab} toolbox for optimization on manifolds},
  author={Boumal, Nicolas and Mishra, Bamdev and Absil, P-A and Sepulchre, Rodolphe},
  journal={The Journal of Machine Learning Research},
  volume={15},
  number={1},
  pages={1455--1459},
  year={2014},
  publisher={JMLR. org}
}

@misc{wright2006numerical,
  title={Numerical optimization},
  author={Wright, Stephen J},
  year={2006}
}

@article{burer2003nonlinear,
  title={A nonlinear programming algorithm for solving semidefinite programs via low-rank factorization},
  author={Burer, Samuel and Monteiro, Renato DC},
  journal={Mathematical programming},
  volume={95},
  number={2},
  pages={329--357},
  year={2003},
  publisher={Springer}
}

@article{teng2024convex,
  title={Convex geometric motion planning of multi-body systems on {Lie} groups via variational integrators and sparse moment relaxation},
  author={Teng, Sangli and Jasour, Ashkan and Vasudevan, Ram and Ghaffari, Maani},
  journal=J-IJRR,
  pages={02783649241296160},
  year={2024},
  publisher={SAGE Publications Sage UK: London, England}
}

@article{wang2023solving,
  title={Solving low-rank semidefinite programs via manifold optimization},
  author={Wang, Jie and Hu, Liangbing},
  journal={arXiv preprint arXiv:2303.01722},
  year={2023}
}

@article{kobilarov2011discrete,
  title={Discrete geometric optimal control on {Lie} groups},
  author={Kobilarov, Marin B and Marsden, Jerrold E},
  journal={IEEE Transactions on Robotics},
  volume={27},
  number={4},
  pages={641--655},
  year={2011},
  publisher={IEEE}
}

@article{kim2019highly,
  title={Highly dynamic quadruped locomotion via whole-body impulse control and model predictive control},
  author={Kim, Donghyun and Di Carlo, Jared and Katz, Benjamin and Bledt, Gerardo and Kim, Sangbae},
  journal={arXiv preprint arXiv:1909.06586},
  year={2019}
}

@article{boutselis2020discrete,
  title={Discrete-time differential dynamic programming on {Lie} groups: Derivation, convergence analysis, and numerical results},
  author={Boutselis, George I and Theodorou, Evangelos},
  journal={IEEE Transactions on Automatic Control},
  volume={66},
  number={10},
  pages={4636--4651},
  year={2020},
  publisher={IEEE}
}

@article{houska2011acado,
  title={{ACADO} toolkit—An open-source framework for automatic control and dynamic optimization},
  author={Houska, Boris and Ferreau, Hans Joachim and Diehl, Moritz},
  journal={Optimal control applications and methods},
  volume={32},
  number={3},
  pages={298--312},
  year={2011},
  publisher={Wiley Online Library}
}

@inproceedings{mueller2013computationally,
  title={A computationally efficient algorithm for state-to-state quadrocopter trajectory generation and feasibility verification},
  author={Mueller, Mark W and Hehn, Markus and D'Andrea, Raffaello},
  booktitle=C-IROS,
  pages={3480--3486},
  year={2013},
  organization={IEEE}
}

@inproceedings{agrawal2022vision,
  title={Vision-aided dynamic quadrupedal locomotion on discrete terrain using motion libraries},
  author={Agrawal, Ayush and Chen, Shuxiao and Rai, Akshara and Sreenath, Koushil},
  booktitle=C-ICRA,
  pages={4708--4714},
  year={2022},
  organization={IEEE}
}

@article{ding2021representation,
  title={Representation-free model predictive control for dynamic motions in quadrupeds},
  author={Ding, Yanran and Pandala, Abhishek and Li, Chuanzheng and Shin, Young-Ha and Park, Hae-Won},
  journal={IEEE Transactions on Robotics},
  volume={37},
  number={4},
  pages={1154--1171},
  year={2021},
  publisher={IEEE}
}

@inproceedings{howell2019altro,
  title={{ALTRO}: A fast solver for constrained trajectory optimization},
  author={Howell, Taylor A and Jackson, Brian E and Manchester, Zachary},
  booktitle=C-IROS,
  pages={7674--7679},
  year={2019},
  organization={IEEE}
}

@article{kalabic2017mpc,
  title={{MPC} on manifolds with an application to the control of spacecraft attitude on {SO(3)}},
  author={Kalabi{\'c}, Uro{\v{s}} V and Gupta, Rohit and Di Cairano, Stefano and Bloch, Anthony M and Kolmanovsky, Ilya V},
  journal={Automatica},
  volume={76},
  pages={293--300},
  year={2017},
  publisher={Elsevier}
}

@article{gill2005snopt,
  title={{SNOPT}: An {SQP} algorithm for large-scale constrained optimization},
  author={Gill, Philip E and Murray, Walter and Saunders, Michael A},
  journal={SIAM review},
  volume={47},
  number={1},
  pages={99--131},
  year={2005},
  publisher={SIAM}
}

@article{zucker2013chomp,
  title={Chomp: Covariant {Hamiltonian} Optimization for motion planning},
  author={Zucker, Matt and Ratliff, Nathan and Dragan, Anca D and Pivtoraiko, Mihail and Klingensmith, Matthew and Dellin, Christopher M and Bagnell, J Andrew and Srinivasa, Siddhartha S},
  journal=J-IJRR,
  volume={32},
  number={9-10},
  pages={1164--1193},
  year={2013},
  publisher={SAGE Publications Sage UK: London, England}
}

@article{manchester2019contact,
  title={Contact-implicit trajectory optimization using variational integrators},
  author={Manchester, Zachary and Doshi, Neel and Wood, Robert J and Kuindersma, Scott},
  journal=J-IJRR,
  volume={38},
  number={12-13},
  pages={1463--1476},
  year={2019},
  publisher={SAGE Publications Sage UK: London, England}
}

@article{li2023autonomous,
  title={Autonomous navigation of underactuated bipedal robots in height-constrained environments},
  author={Li, Zhongyu and Zeng, Jun and Chen, Shuxiao and Sreenath, Koushil},
  journal=J-IJRR,
  volume={42},
  number={8},
  pages={565--585},
  year={2023},
  publisher={SAGE Publications Sage UK: London, England}
}

@article{teng2024gmkf,
  title={{GMKF}: Generalized moment kalman filter for polynomial systems with arbitrary noise},
  author={Teng, Sangli and Zhang, Harry and Jin, David and Jasour, Ashkan and Ghaffari, Maani and Carlone, Luca},
  journal={arXiv preprint arXiv:2403.04712},
  year={2024}
}

@INPROCEEDINGS{Dong-RSS-23, 
    AUTHOR    = {Dayi E Dong AND Henry P Berger AND Ian Abraham}, 
    TITLE     = {{Time Optimal Ergodic Search}}, 
    BOOKTITLE = {Proceedings of Robotics: Science and Systems}, 
    YEAR      = {2023}, 
    ADDRESS   = {Daegu, Republic of Korea}, 
    MONTH     = {July}, 
    DOI       = {10.15607/RSS.2023.XIX.082} 
}

@article{schulman2014motion,
  title={Motion planning with sequential convex optimization and convex collision checking},
  author={Schulman, John and Duan, Yan and Ho, Jonathan and Lee, Alex and Awwal, Ibrahim and Bradlow, Henry and Pan, Jia and Patil, Sachin and Goldberg, Ken and Abbeel, Pieter},
  journal=J-IJRR,
  volume={33},
  number={9},
  pages={1251--1270},
  year={2014},
  publisher={SAGE Publications Sage UK: London, England}
}

@article{lai2024riemannian,
  title={{Riemannian} interior point methods for constrained optimization on manifolds},
  author={Lai, Zhijian and Yoshise, Akiko},
  journal={Journal of Optimization Theory and Applications},
  volume={201},
  number={1},
  pages={433--469},
  year={2024},
  publisher={Springer}
}

@article{liu2020simple,
  title={Simple algorithms for optimization on {Riemannian} manifolds with constraints},
  author={Liu, Changshuo and Boumal, Nicolas},
  journal={Applied Mathematics \& Optimization},
  volume={82},
  number={3},
  pages={949--981},
  year={2020},
  publisher={Springer}
}

@article{yamakawa2022sequential,
  title={Sequential optimality conditions for nonlinear optimization on {Riemannian} manifolds and a globally convergent augmented {Lagrangian} method},
  author={Yamakawa, Yuya and Sato, Hiroyuki},
  journal={Computational Optimization and Applications},
  volume={81},
  number={2},
  pages={397--421},
  year={2022},
  publisher={Springer}
}

@article{obara2022sequential,
  title={Sequential quadratic optimization for nonlinear optimization problems on {Riemannian} manifolds},
  author={Obara, Mitsuaki and Okuno, Takayuki and Takeda, Akiko},
  journal={SIAM Journal on Optimization},
  volume={32},
  number={2},
  pages={822--853},
  year={2022},
  publisher={SIAM}
}

@article{schiela2020sqp,
  title={An {SQP} method for equality constrained optimization on manifolds},
  author={Schiela, Anton and Ortiz, Julian},
  journal={arXiv preprint arXiv:2005.06844},
  year={2020}
}

@article{ratliff2018riemannian,
  title={{Riemannian} motion policies},
  author={Ratliff, Nathan D and Issac, Jan and Kappler, Daniel and Birchfield, Stan and Fox, Dieter},
  journal={arXiv preprint arXiv:1801.02854},
  year={2018}
}

@book{boumal2023introduction,
  title={An introduction to optimization on smooth manifolds},
  author={Boumal, Nicolas},
  year={2023},
  publisher={Cambridge University Press}
}

@article{forster2016manifold,
  title={On-manifold preintegration for real-time visual--inertial odometry},
  author={Forster, Christian and Carlone, Luca and Dellaert, Frank and Scaramuzza, Davide},
  journal={IEEE Transactions on Robotics},
  volume={33},
  number={1},
  pages={1--21},
  year={2016},
  publisher={IEEE}
}

@article{sentis2005synthesis,
  title={Synthesis of whole-body behaviors through hierarchical control of behavioral primitives},
  author={Sentis, Luis and Khatib, Oussama},
  journal={International Journal of Humanoid Robotics},
  volume={2},
  number={04},
  pages={505--518},
  year={2005},
  publisher={World Scientific}
}

@article{khatib1987unified,
  title={A unified approach for motion and force control of robot manipulators: The operational space formulation},
  author={Khatib, Oussama},
  journal={IEEE Journal on Robotics and Automation},
  volume={3},
  number={1},
  pages={43--53},
  year={1987},
  publisher={IEEE}
}

@inproceedings{lee2010geometric,
  title={Geometric tracking control of a quadrotor UAV on {SE(3)}},
  author={Lee, Taeyoung and Leok, Melvin and McClamroch, N Harris},
  booktitle=C-CDC,
  pages={5420--5425},
  year={2010},
  organization={IEEE}
}

@ARTICLE{10301632,
  author={Jang, Junwoo and Teng, Sangli and Ghaffari, Maani},
  journal={IEEE Robotics and Automation Letters}, 
  title={Convex Geometric Trajectory Tracking Using Lie Algebraic MPC for Autonomous Marine Vehicles}, 
  year={2023},
  volume={8},
  number={12},
  pages={8374-8381},
  doi={10.1109/LRA.2023.3328450}}

@inproceedings{teng2021legged,
  title={Legged robot state estimation in slippery environments using invariant extended Kalman filter with velocity update},
  author={Teng, Sangli and Mueller, Mark Wilfried and Sreenath, Koushil},
  booktitle=C-ICRA,
  pages={3104--3110},
  year={2021},
  organization={IEEE}
}

@inproceedings{yu2023fully,
  title={Fully proprioceptive slip-velocity-aware state estimation for mobile robots via invariant kalman filtering and disturbance observer},
  author={Yu, Xihang and Teng, Sangli and Chakhachiro, Theodor and Tong, Wenzhe and Li, Tingjun and Lin, Tzu-Yuan and Koehler, Sarah and Ahumada, Manuel and Walls, Jeffrey M and Ghaffari, Maani},
  booktitle=C-IROS,
  pages={8096--8103},
  year={2023},
  organization={IEEE}
}

@article{vandereycken2013low,
  title={Low-rank matrix completion by {Riemannian} optimization},
  author={Vandereycken, Bart},
  journal={SIAM Journal on Optimization},
  volume={23},
  number={2},
  pages={1214--1236},
  year={2013},
  publisher={SIAM}
}

@INPROCEEDINGS{Teng-RSS-23, 
    AUTHOR    = {Sangli Teng AND Ashkan Jasour AND Ram Vasudevan AND Maani Ghaffari Jadidi}, 
    TITLE     = {{Convex Geometric Motion Planning on {Lie} Groups via Moment Relaxation}}, 
    BOOKTITLE = C-RSS, 
    YEAR      = {2023}, 
    ADDRESS   = {Daegu, Republic of Korea}, 
    MONTH     = {July}, 
    DOI       = {10.15607/RSS.2023.XIX.058} 
}

@article{van2022equivariant,
  title={Equivariant filter (eqf)},
  author={van Goor, Pieter and Hamel, Tarek and Mahony, Robert},
  journal=IEEE_J_AC,
  year={2022},
  publisher={IEEE}
}

@article{ghaffari2022progress,
  title={Progress in Symmetry Preserving Robot Perception and Control through Geometry and Learning},
  author={Ghaffari, Maani and Zhang, Ray and Zhu, Minghan and Lin, Chien Erh and Lin, Tzu-Yuan and Teng, Sangli and Li, Tingjun and Liu, Tianyi and Song, Jingwei},
  journal={Frontiers in Robotics and AI},
  volume={9},
  pages={232},
  year={2022},
  publisher={Frontiers}
}

@article{clark2021nonparametric,
  title={Nonparametric Continuous Sensor Registration},
  author={Clark, William and Ghaffari, Maani and Bloch, Anthony},
  journal={Journal of Machine Learning Research},
  volume={22},
  number={271},
  pages={1--50},
  year={2021}
}

@inproceedings{teng2022lie,
  title={{Lie} algebraic cost function design for control on {Lie} groups},
  author={Teng, Sangli and Clark, William and Bloch, Anthony and Vasudevan, Ram and Ghaffari, Maani},
  booktitle=C-CDC,
  pages={1867--1874},
  year={2022},
  organization={IEEE}
}

@inproceedings{teng2022error,
  title={An error-state model predictive control on connected matrix {Lie} groups for legged robot control},
  author={Teng, Sangli and Chen, Dianhao and Clark, William and Ghaffari, Maani},
  booktitle=C-IROS,
  pages={8850--8857},
  year={2022},
  organization={IEEE}
}

@article{wachter2006implementation,
  title={On the implementation of an interior-point filter line-search algorithm for large-scale nonlinear programming},
  author={W{\"a}chter, Andreas and Biegler, Lorenz T},
  journal={Mathematical programming},
  volume={106},
  pages={25--57},
  year={2006},
  publisher={Springer}
}

@article{marsden1999discrete,
  title={Discrete {Euler-Poincar{\'e}} and {Lie-Poisson} equations},
  author={Marsden, Jerrold E and Pekarsky, Sergey and Shkoller, Steve},
  journal={Nonlinearity},
  volume={12},
  number={6},
  pages={1647},
  year={1999},
  publisher={IOP Publishing}
}

@article{rosen2019se,
  title={{SE-Sync}: A certifiably correct algorithm for synchronization over the special Euclidean group},
  author={Rosen, David M and Carlone, Luca and Bandeira, Afonso S and Leonard, John J},
  journal=J-IJRR,
  volume={38},
  number={2-3},
  pages={95--125},
  year={2019},
  publisher={Sage Publications Sage UK: London, England}
}

@article{posa2014direct,
  title={A direct method for trajectory optimization of rigid bodies through contact},
  author={Posa, Michael and Cantu, Cecilia and Tedrake, Russ},
  journal=J-IJRR,
  volume={33},
  number={1},
  pages={69--81},
  year={2014},
  publisher={Sage Publications Sage UK: London, England}
}

@inproceedings{hereid2017frost,
  title={{FROST}: Fast robot optimization and simulation toolkit},
  author={Hereid, Ayonga and Ames, Aaron D},
  booktitle=C-IROS,
  pages={719--726},
  year={2017},
  organization={IEEE}
}

@article{andersson2019casadi,
  title={{CasADi}: a software framework for nonlinear optimization and optimal control},
  author={Andersson, Joel AE and Gillis, Joris and Horn, Greg and Rawlings, James B and Diehl, Moritz},
  journal={Mathematical Programming Computation},
  volume={11},
  number={1},
  pages={1--36},
  year={2019},
  publisher={Springer}
}

@article{betts1998survey,
  title={Survey of numerical methods for trajectory optimization},
  author={Betts, John T},
  journal={Journal of guidance, control, and dynamics},
  volume={21},
  number={2},
  pages={193--207},
  year={1998}
}

@inproceedings{lee2005lie,
  title={A {Lie} group variational integrator for the attitude dynamics of a rigid body with applications to the {3D} pendulum},
  author={Lee, Taeyoung and McClamroch, N Harris and Leok, Melvin},
  booktitle={Proceedings of IEEE Conference on Control Applications},
  pages={962--967},
  year={2005},
  organization={IEEE}
}

@article{bullo1999tracking,
  title={Tracking for fully actuated mechanical systems: a geometric framework},
  author={Bullo, Francesco and Murray, Richard M},
  journal={Automatica},
  volume={35},
  number={1},
  pages={17--34},
  year={1999},
  publisher={Elsevier}
}

@book{bullo2019geometric,
  title={Geometric control of mechanical systems: modeling, analysis, and design for simple mechanical control systems},
  author={Bullo, Francesco and Lewis, Andrew D},
  volume={49},
  year={2019},
  publisher={Springer}
}

@article{marsden2001discrete,
  title={Discrete mechanics and variational integrators},
  author={Marsden, Jerrold E and West, Matthew},
  journal={Acta Numerica},
  volume={10},
  pages={357--514},
  year={2001},
  publisher={Cambridge University Press}
}

@article{marsden1998introduction,
  title={Introduction to Mechanics and Symmetry},
  author={Marsden, Jerrold E and Ratiu, Tudor S},
  year={1998},
  publisher={Springer}
}

@inproceedings{brudigam2021integrator,
  title={Linear-time variational integrators in maximal coordinates},
  author={Br{\"u}digam, Jan and Manchester, Zachary},
  booktitle={International Workshop on the Algorithmic Foundations of Robotics},
  pages={194--209},
  year={2021},
  organization={Springer}
}

@incollection{bloch2003nonholonomic,
  title={Nonholonomic mechanics},
  author={Bloch, Anthony M},
  booktitle={Nonholonomic mechanics and control},
  pages={207--276},
  year={2003},
  publisher={Springer}
}

@article{barrau2016invariant,
  title={The invariant extended {Kalman} filter as a stable observer},
  author={Barrau, Axel and Bonnabel, Silv{\`e}re},
  journal=IEEE_J_AC,
  volume={62},
  number={4},
  pages={1797--1812},
  year={2016},
  publisher={IEEE}
}

@article{hartley2020contact,
  title={Contact-aided invariant extended {Kalman} filtering for robot state estimation},
  author={Hartley, Ross and Ghaffari, Maani and Eustice, Ryan M and Grizzle, Jessy W},
  journal=J-IJRR,
  volume={39},
  number={4},
  pages={402--430},
  year={2020},
  publisher={SAGE Publications Sage UK: London, England}
}

@article{milnor1976,
    title = {Curvatures of left invariant metrics on {L}ie groups},
    author = {Milnor, John},
    journal = {Advances in Mathematics},
    volume = {21},
    number = {3},
    pages = {293-329},
    year = {1976}
}

@STRING{C-CDC = {Proc. {IEEE} Conf. Decision Control}}

@STRING{C-ICRA = {Proc. {IEEE} Int. Conf. Robot. and Automation}}

@STRING{C-IROS = {Proc. {IEEE}/{RSJ} Int. Conf. Intell. Robots and Syst.}}

@STRING{C-RSS = {Proc. Robot.: Sci. Syst. Conf.}}

@STRING{J-IJRR = {Int. J. Robot. Res.}}

@STRING{C-CDC = {Proceedings of the {IEEE} Conference on Decision and Control}}

@STRING{C-ICRA = {Proceedings of the {IEEE} International Conference on Robotics and Automation}}

@STRING{C-IROS = {Proceedings of the {IEEE}/{RSJ} International Conference on Intelligent Robots and Systems}}

@STRING{C-RSS = {Proceedings of the Robotics: Science and Systems Conference}}

@STRING{J-IJRR = {International Journal of Robotics Research}}

}
\clearpage

\begin{appendices}
\section{Connections on Matrix Lie Groups}
\label{appx:connection}
In this appendix, we discuss the selection of connections on matrix Lie groups. In the Riemannian optimization framework~\citep{boumal2023introduction}, the Levi-Civita connection compatible with the Riemannian metric is the natural choice. While on matrix Lie groups, we can solely use the Lie group structures without the Riemannian structure~\citep{mahony2002geometry}. 

\subsection{Levi-Civita Connection}
$\mathcal{M}$ is a Riemannian manifold if it is equipped with a metric  $\langle\cdot, \cdot \rangle_x: \operatorname{T}_x\mathcal{M} \times \operatorname{T}_x\mathcal{M} \rightarrow \mathbb{R}$ for each tangent space $x \in \mathcal{M}$ and the map $x \rightarrow \langle V_1(x), V_2(x) \rangle$ for vector fields $V_1$ and $V_2$. 


\begin{definition}[Levi-Civita Connection]
\label{def:levi-civita-connection}
Let $(\mathcal{M},\langle\cdot,\cdot\rangle)$ be a Riemannian manifold.
A Levi-Civita connection is an affine connection $\nabla$ on $\mathcal{M}$
that is compatible with the metric $\langle\cdot, \cdot\rangle$ and torsion-free:
\begin{itemize}
    \item $X\langle Y,Z\rangle
        = \langle \nabla_XY,Z\rangle
        + \langle Y,\nabla_XZ\rangle$ 
    \item $\nabla_XY-\nabla_YX=[X,Y],$ 
\end{itemize}
For any $ X,Y,Z\in\mathfrak{X}(\mathcal{M})$.
\end{definition}

Given the retraction map, we have the second-order Taylor expansion on curves:
\begin{definition}[Second-order Retraction]
\label{def:2nd-retraction}
    Consider $c(t)$ as the retraction curve $c(t) = \operatorname{R}_x(tv)$, for $ x \in \mathcal{M}$ and $v \in \operatorname{T}_x\mathcal{M}$. We have the second-order retraction:
\begin{equation}
\label{eq:2nd-curve}
\begin{aligned}
& \quad f\left(\mathrm{R}_x(t v)\right)= \\
& f(x)+t\langle\operatorname{grad} f(x), v\rangle_x 
 +\frac{t^2}{2}\langle \operatorname{Hess} f(x)[v], v\rangle_x \\
& +\frac{t^2}{2}\left\langle\operatorname{grad} f(x), \ddot{c}(0)\right\rangle_x+\mathcal{O}\left(t^3\right) .
\end{aligned}
\end{equation}
\end{definition}

\begin{remark}
\label{rmk:d2c=0}
    In the case that the retraction $\operatorname{R}_x(\cdot) $ is the Riemannian exponential map, the acceleration of $c(t)$, i.e, $\ddot{c}(0) = 0$. If the retraction is available, \eqref{eq:2nd-curve} provides a convenient way to compute the derivatives by extracting the coefficients of $t$ and $t^2$.  
\end{remark}

On a matrix Lie group $\mathcal{G}$, we say a Riemannian metric is left (resp. right) invariant if for $\phi, \eta \in \mathfrak{g}$ and $X \in \mathcal{G}$:
\begin{equation}
\tag{left-invariant metric}
    \langle X\phi^{\wedge}, X\eta^{\wedge} \rangle_{\operatorname{T}_X\mathcal{G}} = \langle \phi^{\wedge}, \eta^{\wedge} \rangle_{\mathfrak{g}},\ \ 
\end{equation}
\begin{equation}
\tag{right-invariant metric}
    \langle \phi^{\wedge}X, \eta^{\wedge}X \rangle_{\operatorname{T}_X\mathcal{G}} = \langle \phi^{\wedge}, \eta^{\wedge} \rangle_{\mathfrak{g}}.
\end{equation}
A metric is \emph{bi-invariant} if it is both left and right invariant. In general, the Lie exponential and Riemannian exponential are not identical, as bi-invariant metrics may not exist for $\mathcal{G}$.
In particular, for a bi-invariant metric to exist, $\mathrm{Ad}_g$ must be an isometry on $\mathfrak{g}$ \citep{milnor1976}, which is always possible when the group is either compact or Abelian.
\begin{remark}
\label{rmk:Lie=Rieman}
If a matrix Lie group \(\mathcal G\) admits a bi-invariant Riemannian metric, such as a compact or Abelian group, the Riemannian exponential at the identity coincides with the Lie exponential.
\end{remark}
For trajectory optimization on $\mathrm{SO}(3)\times \mathbb{R}^3$, the derivation based on the second-order expansion of the Lie exponential is compatible with the bi-invariant Riemannian metric. For $\mathrm{SE}(3)$, the Lie exponential is not compatible, as a bi-invariant metric is absent in a non-compact Lie group.

\subsection{Cartan-Schouten Connection}
Now we introduce the Cartan-Schouten connection, which is intrinsic to matrix Lie groups without a Riemannian structure. 

\begin{definition}[Cartan--Schouten Connections]
\label{def:cartan-schouten-connections}
Let $\mathcal{G}$ be a connected Lie group with Lie algebra $\mathfrak{g}$.
A left-invariant affine connection $\nabla$ on $\mathcal{G}$ is characterized
by a bilinear connection function
\begin{equation}
    \omega:\mathfrak{g}\times\mathfrak{g}\rightarrow\mathfrak{g},
    \qquad
    \nabla_{X^L}Y^L = \left(\mathrm{d}L_g\right)\omega(x,y),
\end{equation}
where $X^L$ and $Y^L$ are the left-invariant vector fields generated by
$x,y\in\mathfrak{g}$. The Cartan--Schouten connections are the three
bi-invariant affine connections whose connection functions are $\omega^{-}(x,y)=0,\ \omega^{0}(x,y)=\frac{1}{2}[x,y],\  \omega^{+}(x,y)=[x,y]$.
\end{definition}

Unlike a general affine connection on a smooth manifold, the Cartan--Schouten
connections are induced by the Lie group structure itself: their definition uses
the Lie algebra bracket and is invariant under group translations. This structure
makes them compatible with the canonical coordinates of the first kind. Indeed,
for the one-parameter subgroup
\begin{equation}
    \gamma_x(t)=\exp(tx),\qquad x\in\mathfrak{g},
\end{equation}
the velocity satisfies $\dot{\gamma}_x(t)=\left(\mathrm{d}L_{\exp(tx)}\right)x$,
and therefore
\begin{equation}
    \nabla_{\dot{\gamma}_x(t)}\dot{\gamma}_x(t)
    =
    \left(\mathrm{d}L_{\exp(tx)}\right)\omega(x,x).
\end{equation}
Since $\omega^{-}(x,x)=0$, $\omega^{0}(x,x)=\frac{1}{2}[x,x]=0$, and
$\omega^{+}(x,x)=[x,x]=0$, every one-parameter subgroup is a geodesic for each Cartan-Schouten connection. 
\section{Implementation Details of \method}
\subsection{Inertia Correction}
\label{sec:inertia}
We summarize the procedure of inertia correction in~\Cref{alg:inertia-correction}. 
\begin{algorithm}[t]
\caption{Inertia Correction for the KKT Systems}
\label{alg:inertia-correction}
\begin{algorithmic}[1]
\footnotesize
\Require KKT matrix $A$, barrier parameter $\mu$, regularization $\delta_w^{\mathrm{last}}$
\Require Parameters $\delta_{\bar w}^{\min}$, $\delta_{\bar w}^{0}$, $\delta_{\bar w}^{\max}$, $k_{\bar w}^{+}$, $k_w^{+}$, $k_w^{-}$, $\delta_{\bar c}$, $k_c$
\Require Projectors $P_x$ and $P_y$ onto the primal and equality-dual subspaces, target negative index $n_y$
\Ensure Corrected matrix $A^c$, updated $\delta_w^{\mathrm{last}}$

\State $\delta_c \gets 0$, \quad $A^c \gets A$
\State $(n^+,n^-,n^0) \gets \operatorname{Inertia}(A)$
\State $\mathrm{ic\_required} \gets (n^- \neq n_y)\ \textbf{or}\ (n^0 \neq 0)$

\If{$n^0 \neq 0$}
    \State $\delta_c \gets \delta_{\bar c}\mu^{k_c}$
\EndIf

\If{$\mathrm{ic\_required}$}
    \If{$\delta_w^{\mathrm{last}} = 0$}
        \State $\delta_w \gets \delta_{\bar w}^{0}$
    \Else
        \State $\delta_w \gets \max\{\delta_{\bar w}^{\min},\, k_w^{-}\delta_w^{\mathrm{last}}\}$
    \EndIf
\EndIf

\While{$\mathrm{ic\_required}$}
    \State $A^c \gets A + \delta_w P_x - \delta_c P_y$
    \State $(n^+,n^-,n^0) \gets \operatorname{Inertia}(A^c)$
    \State $\mathrm{ic\_required} \gets (n^- \neq n_y)\ \textbf{or}\ (n^0 \neq 0)$

    \If{not $\mathrm{ic\_required}$}
        \State $\delta_w^{\mathrm{last}} \gets \delta_w$
        \State \textbf{break}
    \Else
        \If{$\delta_w^{\mathrm{last}} = 0$}
            \State $\delta_w \gets k_{\bar w}^{+}\delta_w$
        \Else
            \State $\delta_w \gets k_w^{+}\delta_w$
        \EndIf
    \EndIf

    \If{$\delta_w > \delta_{\bar w}^{\max}$}
        \State \textbf{break}
    \EndIf
\EndWhile

\State \Return $(A^c,\delta_w^{\mathrm{last}})$
\end{algorithmic}
\end{algorithm}

\subsection{Linear System Solver}
\label{appx:sym-lin-sys}
To exploit sparse symmetric indefinite solvers such as MUMPS, we eliminate the slack-complementarity block as in IPOPT. Define
\[
\hat\xi_3 := \xi_3 - Z^{-1}S\xi_4.
\]
Then the Newton system~\eqref{eq:newton-kkt} becomes
\begin{equation}
\begin{aligned}
\left[\begin{array}{cccc}
H & A_{\mathrm{E}} & A_{\mathrm{I}}  & 0 \\
A_{\mathrm{E}}^* & 0 & 0 & 0 \\
A_{\mathrm{I}}^* & 0 & -ZS^{-1} & I \\
0 & 0 & 0 & Z
\end{array}\right]
\left[\begin{array}{c}
\xi_1 \\
\xi_2 \\
\hat\xi_3 \\
\xi_4
\end{array}\right]
=
-\left[\begin{array}{c}
\zeta_1 \\
\zeta_2 \\
\zeta_3 \\
\zeta_4
\end{array}\right].
\end{aligned}
\end{equation}
Since the last row gives $\xi_4 = -Z^{-1}\zeta_4$, we eliminate $\xi_4$ and obtain the reduced symmetric indefinite system
\begin{equation}
\begin{aligned}
\left[\begin{array}{ccc}
H & A_{\mathrm{E}} & A_{\mathrm{I}} \\
A_{\mathrm{E}}^* & 0 & 0 \\
A_{\mathrm{I}}^* & 0 & -ZS^{-1}
\end{array}\right]
\left[\begin{array}{c}
\xi_1 \\
\xi_2 \\
\hat\xi_3
\end{array}\right]
=
-\left[\begin{array}{c}
\zeta_1 \\
\zeta_2 \\
\hat{\zeta}_3
\end{array}\right],
\end{aligned}
\end{equation}
with $\hat{\zeta}_3:=\zeta_3 - Z^{-1}\zeta_4$.
After solving this system, $\xi_4$ is recovered from $\xi_4=-Z^{-1}\zeta_4$, and $\xi_3$ is obtained by $\xi_3=\hat\xi_3+Z^{-1}S\xi_4$. We note that the above linear system is symmetric and can fully leverage the symmetric indefinite route of MUMPS solver. 

\subsection{Parameters of \method}
\label{appx:ripm-table}
We implement \method using the line-search primal--dual interior-point
algorithm in \Cref{alg:line-search-RIPM}. The main hyper-parameters are listed
in Table~\ref{table:param-RIPM}.

\begin{table}[h]
\centering
\caption{Hyper-parameters used by \method in \Cref{alg:line-search-RIPM}.}
\label{table:param-RIPM}
\begin{tabular}{ccc}
\hline
 Parameter & Notation & Value \\ \hline
Maximal iteration & ${N}_{\max}$ &  $200$ \\ \hline
Maximal iteration in line search & $J_{\max}$ &  $30$ \\ \hline
Tolerance to KKT of \Cref{prob:cons-nlp} & $\epsilon_{\mathrm{tol}}$ &  $10^{-11}$ \\ \hline
Linear decaying rate of $\mu$ & $\kappa_{\mu}$ &  $0.99$ \\ \hline
Superlinear decaying rate of $\mu$ & $\theta_{\mu}$ &  $1.99$ \\ \hline
Minimal fraction to boundary & $\tau_{\min}$ &  $0.995$ \\ \hline
Barrier problem cost progress & $\gamma_{\theta}$ &  $10^{-6}$ \\ \hline
Minimal infeasibility & $\theta_{\min}$ &  $10^{-4}$ \\ \hline
Progress of barrier cost & $\eta_{\varphi}$ &  $10^{-4}$ \\ \hline
Progress of feasibility & $\gamma_\theta$ &  $10^{-4}$ \\ \hline
Decaying rate for line search & $\beta$ & $0.5$ \\ \hline  
\end{tabular}
\end{table}

\section{Local Convergence of \method}
\label{appx:proof-convergence}
We follow the proof in~\citep{lai2024riemannian} to show the local convergence of \method under standard regularity conditions. 
\localconvergence*
\begin{proof}
For simplicity, write $x:=x^\star$, $y:=y^\star$, $z:=z^\star$, $s:=s^\star$, and let
\begin{equation}
\xi_w=(\xi_x,\xi_y,\xi_z,\xi_{s})
\in T_x\mathcal M\times \mathbb R^l\times \mathbb R^m\times \mathbb R^m.
\end{equation}
We show that
\begin{equation}
J_{\zeta_0}(w^\star)\xi_w=0
\end{equation}
implies
\begin{equation}
\xi_w=0.
\end{equation}

Expanding the linear system \(J_{\zeta_0}(w^\star)\xi_w=0\) using \eqref{eq:kkt-jac} gives
\begin{equation}
\label{eq:kernel-system-1}
\left\{
\begin{aligned}
0 &= H(x,y,z)[\xi_x] + A_E(x)[\xi_y] + A_I(x)[\xi_z],\\
0 &= A_E(x)^*[\xi_x],\\
0 &= A_I(x)^*[\xi_x] + \xi_{s},\\
0 &= S\xi_z + Z\xi_{s} .
\end{aligned}
\right.
\end{equation}
Equivalently, in component form,
\begin{equation}
\label{eq:kernel-system-2}
\left\{
\begin{aligned}
0 &= H(x,y,z)[\xi_x]
   + \sum_{i\in\mathcal{E}}\xi_{y,i}\, \mathrm D h_i(x) \\
  &+ \sum_{j\in\mathcal{I}}\xi_{z,j}\, \mathrm D g_j(x),\\
0 &= \mathrm D h_i(x)[\xi_x], \quad \forall i\in\mathcal{E},\\
0 &= \mathrm D g_j(x)[\xi_x] + \xi_{s,j}, \quad \forall j\in\mathcal{I},\\
0 &= s_j \xi_{z,j} + z_j \xi_{s,j}, \quad \forall j\in\mathcal{I}.
\end{aligned}
\right.
\end{equation}

Since \(w^\star\) satisfies strict complementarity, we have $z_j>0$ for all $j\in\mathcal{A}$ and $s_j>0$ for all $j\in\mathcal{I}\setminus\mathcal{A}$. Hence the last equation in \eqref{eq:kernel-system-2} implies
$\xi_{s,j}=0, \forall j\in\mathcal{A},$ and $\xi_{z,j}=0, \forall j\in\mathcal{I}\setminus\mathcal{A}.
$
Substituting these relations into \eqref{eq:kernel-system-2} yields
\begin{equation}
\label{eq:reduced-system}
\left\{
\begin{aligned}
0 &= H(x,y,z)[\xi_x]
   + \sum_{i\in\mathcal{E}}\xi_{y,i}\, \mathrm D h_i(x) \\
& + \sum_{j\in\mathcal{A}}\xi_{z,j}\, \mathrm D g_j(x),\\
0 &= \mathrm D h_i(x)[\xi_x], \quad \forall i\in\mathcal{E},\\
0 &= \mathrm D g_j(x)[\xi_x], \quad \forall j\in\mathcal{A},
\end{aligned}
\right.
\end{equation}
and $\xi_{s,j}=-\,\mathrm D g_j(x)[\xi_x], \forall j\in\mathcal{I}\setminus\mathcal{A}.$

Evaluating the first equation of \eqref{eq:reduced-system} on \(\xi_x\) gives
\begin{equation}
\begin{aligned}
0
&=
H(x,y,z)[\xi_x,\xi_x]
+\sum_{i\in\mathcal{E}}\xi_{y,i}\,\mathrm D h_i(x)[\xi_x] \\
&+\sum_{j\in\mathcal{A}}\xi_{z,j}\,\mathrm D g_j(x)[\xi_x].
\end{aligned}
\end{equation}
Using the second and third equations of \eqref{eq:reduced-system}, we have
\begin{equation}
0=H(x,y,z)[\xi_x,\xi_x].
\end{equation}
Moreover, \eqref{eq:reduced-system} implies
\begin{equation}
\mathrm D h_i(x)[\xi_x]=0,\quad \forall i\in\mathcal{E},
\end{equation}
\begin{equation}
\mathrm D g_j(x)[\xi_x]=0,\quad \forall j\in\mathcal{A}.
\end{equation}
Therefore, by SOSC,
\begin{equation}
\xi_x=0.
\end{equation}

With \(\xi_x=0\), the first equation of \eqref{eq:reduced-system} becomes
\begin{equation}
\sum_{i\in\mathcal{E}}\xi_{y,i}\, \mathrm D h_i(x)
+\sum_{j\in\mathcal{A}}\xi_{z,j}\, \mathrm D g_j(x)=0.
\end{equation}
By LICQ, the covectors
\begin{equation}
\{\mathrm D h_i(x)\}_{i\in\mathcal{E}}\cup \{\mathrm D g_j(x)\}_{j\in\mathcal{A}}
\end{equation}
are linearly independent. Hence
\begin{equation}
\xi_{y,i}=0,\quad \forall i\in\mathcal{E},
\end{equation}
\begin{equation}
\xi_{z,j}=0,\quad \forall j\in\mathcal{A}.
\end{equation}
Together with
\begin{equation}
\xi_{z,j}=0,\quad \forall j\in\mathcal{I}\setminus\mathcal{A},
\end{equation}
we obtain
\begin{equation}
\xi_z=0.
\end{equation}
Finally, from
\begin{equation}
\xi_{s,j}=-\,\mathrm D g_j(x)[\xi_x]
\end{equation}
and \(\xi_x=0\), we get
\begin{equation}
\xi_{s}=0.
\end{equation}
Thus \(\xi_w=(\xi_x,\xi_y,\xi_z,\xi_{s})=0\), so the kernel of \(J_{\zeta_0}(w^\star)\) is trivial. Therefore, $J_{\zeta_0}(w^\star)$ is non-singular.

Since \(\zeta_\mu\) is smooth in \((w,\mu)\) and
$J_{\zeta_\mu}(w^\star)=J_{\zeta_0}(w^\star)$
the Jacobian \(J_{\zeta_\mu}(w)\) remains nonsingular for all \(w\) in a neighborhood of
\(w^\star\) and all \(\mu\) sufficiently small. Therefore, for each fixed \(\mu\), the Newton step
\begin{equation}
J_{\zeta_\mu}(w_k)\xi_k=-\zeta_\mu(w_k)
\end{equation}
is locally well defined. By the local convergence theorem for Newton's method on manifolds
with a retraction update \citep{absil2008optimization}, the iteration $w_{k+1}=R_{w_k}(\xi_k)
$
is locally superlinearly convergent to the zero of \(\zeta_\mu\), and quadratically convergent
under the stated smoothness assumptions.

It remains to pass from the fixed-\(\mu\) subproblem to the full homotopy iteration. Write
\begin{equation}
\zeta_{\mu_k}(w_k)=\zeta_0(w_k)-\begin{bmatrix}0\\0\\0\\ \mu_k e\end{bmatrix}.
\end{equation}
Hence the Newton step for \(\zeta_{\mu_k}\) is a perturbed Newton step for \(\zeta_0\). By Taylor
expansion of \(\zeta_0\) around \(w^\star\) and the local bounded invertibility of \(J_{\zeta_0}\),
there exist constants \(c_1,c_2>0\) such that
\begin{equation}
\|w_{k+1}-w^\star\|
\le
c_1\|w_k-w^\star\|^2+c_2\mu_k
\end{equation}
for all \(w_k\) sufficiently close to \(w^\star\). Therefore, if
\begin{equation}
\mu_k=o(\|\zeta_0(w_k)\|),
\end{equation}
then the perturbation term is asymptotically negligible relative to the Newton contraction,
and \(\{w_k\}\) converges superlinearly to \(w^\star\). If, moreover,
\begin{equation}
\mu_k=\mathcal O(\|\zeta_0(w_k)\|^2),
\end{equation}
then the perturbation is of higher order and the recursion becomes quadratic, yielding
\begin{equation}
\|w_{k+1}-w^\star\|=\mathcal O(\|w_k-w^\star\|^2).
\end{equation}

Finally, by continuity of the KKT Jacobian and its inertia, once \(w_k\) is sufficiently close
to \(w^\star\), the inertia correction is inactive. In the same neighborhood, the line-search
accepts the full Newton step, so the method reduces locally to the pure primal-dual Newton
iteration. This completes the proof.
\end{proof}

\section{Implementation Details of \method and Baselines}
\label{appx:baseline}
In this appendix, we summarize the implementation settings and convergence
criteria used for \method and the baseline solvers.

\subsection{Equivalence Between the Matrix and Quaternion Attitude Costs}
\label{appx:q=R}

Let
\begin{equation}
\delta q=q_{\mathrm{ref}}^{-1}\otimes q
=
[\delta q_w,\delta q_v^\top]^\top
\end{equation}
be the relative unit quaternion. Then
\begin{equation}
R(\delta q)=R_{\mathrm{ref}}^\top R.
\end{equation}
For any unit quaternion \(q=[q_w,q_v^\top]^\top\), the corresponding rotation
matrix satisfies
\begin{equation}
R(q)
=
(q_w^2-q_v^\top q_v)I
+
2q_vq_v^\top
+
2q_w[q_v]_\times .
\end{equation}
Taking the trace gives
\begin{equation}
\operatorname{tr}(R(q))
=
3(q_w^2-\|q_v\|^2)+2\|q_v\|^2
=
4q_w^2-1,
\end{equation}
where we used \(q_w^2+\|q_v\|^2=1\). Therefore,
\begin{equation}
\begin{aligned}
\|R_{\mathrm{ref}}^\top R-I\|_F^2
&=
\|R(\delta q)-I\|_F^2 \\
&=
6-2\operatorname{tr}(R(\delta q)) \\
&=
6-2(4\delta q_w^2-1) \\
&=
8(1-\delta q_w^2) \\
&=
8\|\delta q_v\|^2 .
\end{aligned}
\end{equation}
Hence, choosing \(Q_q=8 I_3\) yields
\begin{equation}
\delta q_v^\top Q_q\delta q_v
=
8\|\delta q_v\|^2
=
\|R_{\mathrm{ref}}^\top R-I\|_F^2.
\end{equation}

\subsection{\method}
A run is counted as successfully converged when the convergence criterion of
\Cref{prob:cons-nlp} in \eqref{eq:converge} satisfies
\[
    \|E_0\| \leq \epsilon_{\mathrm{tol}} = 10^{-6}.
\]
We note that \method does not include a full filter mechanism or restoration
phase, in contrast to IPOPT~\citep{wachter2006implementation}; therefore, we do
not claim global convergence for the implemented globalization strategy.

\subsection{IPOPT}
IPOPT~\citep{wachter2006implementation} is the closest baseline to \method,
since both methods are based on primal--dual interior-point iterations. We use
CASADi~\citep{andersson2019casadi} as the modeling and solver interface for
IPOPT, and keep the IPOPT parameters at their default values unless otherwise
specified.

A run is counted as successfully converged when IPOPT reports a KKT residual below
\[
    \epsilon_{\mathrm{tol}} = 10^{-6}.
\]
which is the same as \method.
The maximal number of iterations is set to \(1000\), which includes both the main iterations and the feasibility restoration iterations.

\subsection{SNOPT}
SNOPT~\citep{gill2005snopt} is used as the active-set SQP baseline. We use
CASADi~\citep{andersson2019casadi} as the modeling and solver interface for
SNOPT.

For all SNOPT runs, we set the major iteration limit, i.e., the total iterations of QPs to \(10^4\), and set the
major feasibility $\epsilon_r$ and optimality tolerances $\epsilon_d$ to
\[
    \epsilon_r=\epsilon_d=10^{-6}.
\]
SNOPT solves problems in the form
\[
    \min_x f_0(x),\qquad 
    l \leq 
    \begin{bmatrix}
        x \\ f(x) \\ A_Lx
    \end{bmatrix}
    \leq u ,
\]
where \(f(x)\) denotes the nonlinear constraint rows. The major feasibility
tolerance controls the accuracy of these nonlinear rows. Specifically, SNOPT
requires the normalized maximum nonlinear constraint violation
\[
    \mathrm{rowerr}
    =
    \max_i \frac{\mathrm{viol}_i}{\|x\|}
    \leq \epsilon_r ,
\]
with $    \mathrm{viol}_i
    =
    \max\{l_i-f_i(x),\,0,\,f_i(x)-u_i\}.
$. Thus, this tolerance directly controls the accuracy of the nonlinear dynamics
and path constraints in our trajectory optimization problems.

The major optimality tolerance controls SNOPT's dual optimality measure. Let $d_j = g_j-\pi^\top a_j$ be the reduced gradient associated with the \(j\)th variable or slack, where
\(g_j\) is the corresponding objective-gradient component, \(a_j\) is the
corresponding column of \((A,-I)\), and \(\pi\) is the QP multiplier vector.
SNOPT defines the complementarity estimate
$$        \mathrm{Comp}_j =
    \begin{cases}
    d_j\min\{x_j-l_j,1\}, & d_j\geq 0,\\
    -d_j\min\{u_j-x_j,1\}, & d_j<0,
    \end{cases}
$$
and requires $    \mathrm{max}\ \mathrm{Comp}
    =
    \max_j \frac{\mathrm{Comp}_j}{\|\pi\|}
    \leq \epsilon_d .
$
Therefore, the optimality tolerance should be interpreted as SNOPT's
reduced-gradient complementarity accuracy, rather than as a direct bound on
the full Lagrangian-gradient residual. The minor and total iteration limits are
left at their default values.

\subsection{ACADO}
ACADO is used as the multiple-shooting SQP baseline. After discretization, the
optimal control problem is treated as a finite-dimensional nonlinear program
with Lagrangian
\[
    \mathcal{L}(w,\lambda,\nu)
    =
    J(w)+\lambda^\top c(w)+\nu^\top g(w).
\]
Therefore, convergence is measured by the KKT conditions of this discretized
problem, namely stationarity, primal feasibility, and complementarity:
\[
    \nabla_w \mathcal{L}(w,\lambda,\nu)\approx 0,\qquad
    c(w)\approx 0,\qquad
    0\leq \nu \perp g(w)\leq 0 .
\]
Following the default ACADO setting, we use the KKT tolerance
\[
    \epsilon_{\mathrm{KKT}} = 10^{-6}.
\]
A run is counted as successfully converged only when the KKT tolerance reported
by ACADO is below this threshold.

\subsection{CROCODDYL}
CROCODDYL~\citep{mastalli2020crocoddyl} is used as the multi-shooting DDP
baseline. In particular, we use its box-constrained feasibility-driven DDP (FDDP)
implementation.

A run is counted as successfully converged when the internal FDDP stopping
measure is below \(5\times 10^{-5}\), the internal feasibility condition is
satisfied, and the final reported primal feasibility satisfies
\[
    \mathrm{\epsilon}_{\mathrm{primal}} < 10^{-6}.
\]
The FDDP globalization strategy contracts the shooting defects according to
\[
    \bar f^{+}=(1-\alpha)\bar f,
\]
and accepts steps using a Goldstein-type decrease test.

\subsection{ALTRO}
ALTRO is used as the augmented-Lagrangian single-shooting baseline based on an iterative Linear Quadratic Regulator (iLQR). A run is counted
as successfully converged when the solver returns a successful status, and the recomputed objective value is finite.

The maximal number of iterations is set to \(500\). The configured stopping
criteria use primal feasibility below \(10^{-6}\), cost decrease below \(10^{-4}\). 

\end{appendices}

\end{document}